\documentclass[twoside,11pt]{article}

\usepackage{graphicx} 
\usepackage{amsmath,amsfonts,amssymb}
\usepackage{cite,url,algorithm}
\usepackage{color,subfigure}
\usepackage{multirow}
\usepackage{rotating}

\graphicspath{{figures/}}

\usepackage{enumitem}

\definecolor{red}{rgb}{0.738,0,0}

\newcommand{\rr}{{\mathbb R}}
\newcommand{\ba}[1]{\begin{array}{#1}}
\newcommand{\ea}{\end{array}}
\newcommand{\XX}{\mathcal{X}}
\newcommand{\st}{\mathop{\rm s.t.}\nolimits}
\ifx\thelemma\undefined
	\newtheorem{lemma}{Lemma}
\fi
\newcommand{\diag}{\mathop{\rm diag}\nolimits}

\newcommand{\pref}{\pi}

\newcommand{\Ss}[1]{\mathcal{S}_{#1}}
\definecolor{blue}{rgb}{0,0,0.738}

\newcommand{\prob}{\Pr}
\newtheorem{theorem}{Theorem} 
\newenvironment*{proof}[1]{\textbf{\emph{Proof}} }{}
\newcommand{\matrice}[2]{\left[\hspace*{-.1cm}\ba{#1} #2 \ea\hspace*{-.1cm}\right]}

\begin{document}

\title{Active preference learning based on radial basis functions}

\author{Alberto Bemporad$^\dagger$ \and Dario Piga$^\star$}
\date{{\small
    $^\dagger$IMT School for Advanced Studies Lucca, Italy \\Email: \texttt{alberto.bemporad@imtlucca.it}\\\vspace*{1em}
    $^\star$Dalle Molle Institute for Artificial Intelligence - USI/SUPSI, Manno, Switzerland \\
    Email: \texttt{dario.piga@supsi.ch}}\\[3em]%
    \today
}

\maketitle

\begin{abstract}
This paper proposes a method for solving optimization problems in which the decision-maker 
cannot evaluate the objective function, but rather can only express a 
\emph{preference} such as ``this is better than that''
between two candidate decision vectors.
The algorithm described in this paper aims at reaching the global optimizer
by iteratively proposing the decision maker a new comparison to make, 
based on actively learning  a surrogate of the latent (unknown and perhaps unquantifiable) objective function
from past sampled decision vectors and pairwise preferences. 
The surrogate is fit by means of radial basis functions,
under the constraint of satisfying, if possible, the preferences expressed by the decision
maker on existing samples. The surrogate is used to propose a new sample of the decision vector
for comparison with the current best candidate based on two possible
criteria: minimize a combination of the surrogate
and an inverse weighting distance function to balance between exploitation
of the surrogate and exploration of the decision space, or 
maximize a function related to the probability that the new candidate will be preferred. Compared to active preference learning
based on Bayesian optimization, we show that our approach is superior
in that, within the same number of comparisons,
it approaches the global optimum more closely and is computationally lighter.
MATLAB and a Python implementations of the algorithms described in the paper 
are available at \url{http://cse.lab.imtlucca.it/~bemporad/idwgopt}.
\end{abstract}


\section{Introduction}
Taking an optimal decision is the process of selecting the value of certain variables that produces ``best'' results. 
When using mathematical programming to solve this problem, ``best'' means that the taken decision minimizes a certain cost function,
or equivalently maximizes a certain utility function. However, in many problems an objective function is not quantifiable, either because it is of qualitative nature or because it involves several goals. Moreover, sometimes the ``goodness'' of a certain combination of decision variables 
can only be assessed by a human decision maker. 

This situation arises in many practical cases. When calibrating the parameters of a deep neural network whose goal is to generate a synthetic painting or artificial music, artistic ``beauty'' is hardly captured by a numerical function, 
and a human decision-maker is required to assess whether a certain combination of parameters produces ``beautiful'' results.
For example, the authors of~\cite{BDG08} propose a tool to help digital artists to calibrate the parameters of an image generator so
that the synthetic image ``resembles'' a given one. Another example is in 
industrial automation when a calibrating the tuning knobs of a control system:
based on engineering insight and rules of thumb, the task is usually carried out manually 
by trying a series of combinations until the calibrator is satisfied by the observed closed-loop performance. 
Another frequent situation in which it is hard to formulate an objective function is in multi-objective optimization~\cite{CP07}.
Here, selecting a-priori the correct weighted sum of the objectives to minimize 
in order to choose an optimal decision vector can be very difficult, and is often a human operator that needs to assess whether a certain Pareto optimal solution is better than another one, based on his or her (sometimes unquantifiable) feelings. 

It is well known in neuroscience that humans are better at choosing between two options (``this is better than that'') than among multiple ones~\cite{CKHWR14,CBG15}. In consumer psychology, the ``choice overload'' effect shows that a human, when presented an abundance of options,
has more difficulty to make a decision than if only a few options are given. On the other hand, having  a large number of possibilities
to choose from creates very positive feelings in the decision maker~\cite{CBG15}. In economics, the difficulty of rational behavior in choosing the best option
was also recognized in~\cite{Sim55}, due to the complexity of the decision problem exceeding the cognitive resources of the decision maker. Indeed, choosing the best option implies ranking choices by absolute values and therefore quantifying a clear objective function to optimize, a process that
might be difficult due to the complexity and fuzziness of many criteria that are involved. For the above reasons, the importance of focusing on discrete choices
in psychology dates back at least to the 1920's~\cite{Thu27}.

Looking for an optimal value of the decision variables that is ``best'', in that the human operator always prefers it compared to
all other tested combinations, may involve a lot of trial and error. For example, in parameter calibration the operator has to try
many combinations before being satisfied with the winner one. The goal of \emph{active preference learning} is to drive the trials
by automatically proposing decision vectors to the operator for testing, so to converge to the best choice possibly within the least number of experiments.

In the derivative-free black-box global optimization literature there exist some methods for minimizing  an objective function $f$ 
that can be used also for preference-based learning. 
Since, given two decision vectors $x_1$, $x_2$, we can say that $x_2$ is not ``preferred'' to $x_1$ if $f(x_1)\leq f(x_2)$, finding
a global optimizer can be reinterpreted as the problem of looking for the vector $x^\star$ such that it is preferred
to any other vector $x$. Therefore, optimization methods that only observe the outcome of the comparison $f(x_1)\leq f(x_2)$ (and not the values $f(x_1),f(x_2)$, not even the difference $f(x_1)-f(x_2)$) can be used for preference-based optimization. For example
particle swarm optimization (PSO) algorithms~\cite{Ken10,VV07} drive the evolution of particles only
based on the outcome of comparisons between function values and could be used in principle for preference-based optimization. However, although very effective in solving many complex global optimization problems, PSO is not conceived for keeping the number of evaluated preferences small, as it relies on randomness (of changing magnitude) in moving the particles, and would be therefore especially inadequate in solving problems where a human
decision maker is involved in the loop to express preferences.

Different methods were proposed in the global optimization literature for 
finding a global minimum of  functions that are expensive to evaluate~\cite{RS13}. Some of the most successful ones rely on
computing a simpler-to-evaluate \emph{surrogate} of the objective function and use it to drive the search
of new candidate optimizers to sample~\cite{Jon01}. The surrogate is refined iteratively as new values of the actual objective function  are collected
at those points. Rather than minimizing the surrogate, which may easily lead to miss the global optimum
of the actual objective function, an \emph{acquisition function} is minimized instead to generate new candidates.
The latter function consists of a combination of the surrogate and of an extra term that promotes
exploring areas of the decision space that have not been yet sampled. 

Bayesian Optimization (BO) is a very popular method exploiting surrogates to globally optimize functions
that are expensive to evaluate. In BO, the surrogate of the underlying objective function is modeled as a Gaussian process (GP),
so that model uncertainty can be characterized using probability theory and used to drive the search~\cite{Kus64}.
BO is used in several methods such as Kriging~\cite{Mat63}, in Design and Analysis of Computer Experiments (DACE)~\cite{SWMW89}, 
in the Efficient Global Optimization (EGO) algorithm~\cite{JSW98}, and is nowadays heavily used in machine learning for hyper-parameter tuning~\cite{BCD10}.

Bayesian optimization has been proposed also for minimizing (unknown) black-box functions based only on preferences~\cite{BDG08,GZDL17,ASRGV19}. The surrogate function describing the observed set of preferences is described in terms of a GP, using a probit model to describe the observed pairwise preferences~\cite{CG05} and the Laplace approximation  of the posterior distribution of the latent function to minimize.  The GP provides a probabilistic prediction of the preference that  is used to define an acquisition function (like expected improvement) which is maximized in order to select the next query point. The acquisition function used in Bayesian preference automatically balances \emph{exploration} (selecting  queries  with high uncertainty on the preference) and \emph{exploitation} (selecting  queries which are expected to lead to improvements in the objective function).

In this paper we propose a new approach to active preference learning optimization that models the surrogate by using
  general radial basis functions (RBFs) rather than a GP, and inverse distance weighting (IDW) functions for exploration
of the space of decision variables. A related approach was recently proposed by one of the authors in~\cite{Bem19}
for global optimization of known, but difficult to evaluate, functions. Here, we use instead RBFs to construct a surrogate function that only
needs to satisfy, if possible, the preferences already expressed by the decision maker at sampled points. 
The  training dataset of the surrogate function is actively augmented in an incremental way by the proposed algorithm
according to two alternative criteria. The first criterion, similarly to~\cite{Bem19}, is based on a trade off
between minimizing the surrogate and  maximizing the distance from existing samples using IDW functions.
At each iteration, the RBF weights are computed by solving a linear or  quadratic programming  problem aiming at satisfying 
the available training  set of pairwise preferences.  The second alternative criterion is based on quantifying the probability of getting an improvement based on a maximum-likelihood  interpretation of the RBF weight selection  problem, which allows 
quantifying the probability of getting an improvement based on the surrogate function. Based on one of the above criteria, 
the proposed algorithm constructs an acquisition function that is very cheap to evaluate and is minimized
to generate a new sample and to query a new preference.

Compared to  preferential  Bayesian optimization, the proposed approach
is computationally lighter, due to the fact that computing the surrogate simply requires solving a convex quadratic or linear programming problem. Instead, in PBO, one has  to first compute the Laplace approximation  of the posterior distribution of the preference function, which requires to  calculate (via a Newton-Raphson numerical optimization algorithm) the mode of the posterior distribution. Then, a system of linear equations, with size equal to the number of observations, has to be  solved. 	 
Moreover, the IDW term used by our approach to promote exploration does not depend on the surrogate, which guarantees that the space of optimization variables is well explored  even if the surrogate poorly approximates the underlying preference function.     
Finally, the performance of our method in approaching the optimizer within the allocated number of preference queries is similar  and sometimes better than preferential Bayesian optimization, as we will show in a set of benchmarks used in global optimization and in solving a multi-objective
optimization problem.  

The paper is organized as follows. In Section~\ref{sec:prob_formulation} we formulate the preference-based optimization problem we want to solve. Section~\ref{sec:surrogate}
proposes the way to construct the surrogate function using linear or quadratic programming
and Section~\ref{sec:acquisition} the acquisition functions that are used for generating new samples. The active preference learning algorithm is stated in Section~\ref{sec:algorithm} and its
possible application to solve multi-objective optimization problems in Section~\ref{sec:multiobj}.
Section~\ref{sec:results} presents numerical results obtained in applying the preference learning
algorithm for solving a set of benchmark global optimization problems, a multi-objective optimization problem, and for optimal tuning of a cost-sensitive  neural network classifier  for object recognition from images. Finally, some conclusions are drawn in Section~\ref{sec:conclusions}.

A MATLAB and a Python implementation of the proposed approach is available for download at \url{http://cse.lab.imtlucca.it/~bemporad/idwgopt}.

\section{Problem statement}
\label{sec:prob_formulation}
Let $\rr^n$ be the space of decision variables.
Given two possible decision vectors $x_1,x_2\in\rr^n$, consider
the \emph{preference function}  $\pref:\rr^n\times\rr^n\to\{-1,0,1\}$ defined as
\begin{equation}
    \pref(x_1,x_2)=\left\{\ba{ll}
    -1 &\mbox{if $x_1$ ``better'' than $x_2$}\\
    0 &\mbox{if $x_1$ ``as good as'' $x_2$}\\
    1 &\mbox{if $x_2$ ``better'' than $x_1$}
    \ea\right.
\label{eq:pref_fun}
\end{equation}
where for all $x_1,x_2\in\rr^n$ it holds $\pref(x_1,x_1)=0$, $\pref(x_1,x_2)=-\pref(x_2,x_1)$,
and the transitive property
\begin{equation}
    \pref(x_1,x_2) = \pref(x_2,x_3)=-1\ \Rightarrow  \pref(x_1,x_3) = -1
\label{eq:trans-prop}
\end{equation}
The objective of this paper is to solve the following
constrained global optimization problem:
\begin{equation}
    \mbox{find}\ x^\star\ \mbox{such that}\ \pref(x^\star,x)\leq 0,\ \forall x\in\XX,\ \ell\leq x\leq u
\label{eq:glob-opt-pref}
\end{equation}
that is to find the vector $x^\star\in\rr^n$ of decision variables that is ``better'' (or ``no worse'')
than any other vector $x\in\rr^n$ according to the preference function $\pref$.

Vectors  $\ell,u\in\rr^n$ in~\eqref{eq:glob-opt-pref} define lower and upper bounds
on the decision vector, and $\XX\subseteq\rr^n$ imposes further constraints on $x$,
such as
\begin{equation}
    \XX=\{x\in\rr^n:\ g(x)\leq 0\}
\label{eq:g(x)}
\end{equation}
where $g:\rr^n\to\rr^q$, and $\XX=\rr^n$ when $q=0$ (no inequality constraint is enforced). We assume that the 
condition $x\in\XX$ 
is easy to evaluate, for example in case of linear inequality constraints 
we have $g(x)=Ax-b$, $A\in\rr^{q\times n}$, $b\in\rr^q$, $q\geq 0$. 
When formulating~\eqref{eq:glob-opt-pref} we have excluded equality
constraints $A_{e}x=b_e$, as they can be eliminated by reducing 
the number of optimization variables.

The problem of minimizing  an \emph{objective function} $f:\rr^n\to\rr$ 
under constraints,
\begin{equation}
\ba{rcrl}
x^\star&=&\arg\min_x& f(x)\\
& &\st & \ell\leq x\leq u\\
& &    &x\in\XX 
\ea
\label{eq:glob-opt}
\end{equation}
can be written as in~\eqref{eq:glob-opt-pref} by defining 
\begin{equation}
    \pref(x_1,x_2)=\left\{\ba{ll}
    -1 &\mbox{if $f(x_1)<f(x_2)$}\\
    0 &\mbox{if $f(x_1)=f(x_2)$}\\
    1 &\mbox{if $f(x_1)>f(x_2)$}
    \ea\right.
\label{eq:pref_fun-f}
\end{equation}
In this paper we assume that \emph{we do not have a way to evaluate the objective function} $f$.
The only assumption we make is that for each given pair
of decision vectors $x_1,x_2\in\XX$, $\ell\leq x\leq u$, only the value $\pref(x_1,x_2)$ is observed. The rationale of our problem formulation is that 
often one encounters practical decision problems in which a function $f$ is impossible to 
quantify, but anyway it is possible to express a \emph{preference}, for example by a human operator, for any given presented pair  $(x_1, x_2)$. The goal
of the active preference learning algorithm proposed in this paper is
to suggest iteratively a sequence of samples $x_1,\ldots,x_N$ to test and compare 
such that $x_N$ approaches $x^\star$ as $N$ grows.

In what follows we implicitly assume that a function $f$ actually exists but is completely unknown,
and attempt to synthesize a \emph{surrogate function} $\hat f:\rr^n\to\rr$ of $f$
such that its associated preference function $\hat\pref:\rr^n\times\rr^n\to\{-1,0,1\}$
defined as in~\eqref{eq:pref_fun-f} coincides with $\pref$ on the finite set of
sampled pairs of decision vectors.

\section{Surrogate function}
\label{sec:surrogate}
Assume that we have generated $N\geq 2$ samples $X=\{x_1\ \ldots\ x_N\}$ of the decision vector, 
with $x_i,x_j\in\rr^n$ such that $x_i\neq x_j$, $\forall i\neq j$,
$i,j=1,\ldots,N$, and have evaluated a \emph{preference vector} $B=[b_1\ \ldots\ b_{M}]'\in\{-1,0,1\}^{M}$
\begin{equation}
        b_h=\pref(x_{i(h)},x_{j(h)})
\label{eq:pref-vector}
\end{equation}
where  $M$ is the number of expressed preferences,  $1 \leq M \leq \binom{N}{2}$, 
$h\in\{1,\ldots,M\}$, $i(h), j(h)\in \{1,\ldots,N\}$, $i(h) \neq j(h)$.

In order to find a surrogate function $\hat f:\rr^{n} \to\rr$ such that
\begin{equation}
   \pref(x_{i(h)},x_{j(h)}) = \hat\pref(x_{i(h)},x_{j(h)}),\ \forall h=1,\ldots,M
\label{eq:surrogate-condition}
\end{equation}
where $\hat \pref$ is defined from $\hat f$ as in~\eqref{eq:pref_fun-f},
we consider a surrogate function $\hat f$ defined as the following
radial basis function (RBF) interpolant~\cite{Gut01,MGTM07}
\begin{equation}
    \hat f(x)=\sum_{i=1}^N\beta_i\phi(\epsilon d(x,x_i))
\label{eq:rbf}
\end{equation}
In~\eqref{eq:rbf} function $d:\rr^{2n}\to\rr$ is the Euclidean distance
\begin{equation}
    d(x_1,x_2)=\|x_1-x_2\|_2^2,\ x_1,x_2\in\rr^n
\label{eq:distance}
\end{equation}
$\epsilon>0$ is a scalar parameter, $\phi:\rr\to\rr$ is a RBF,
and $\beta_i$ are coefficients that we determine as explained below.
Examples of RBFs are $\phi(\epsilon d)=\frac{1}{1+(\epsilon d)^2}$
(\emph{inverse quadratic}), $\phi(\epsilon d)=e^{-(\epsilon d)^2}$ (\emph{Gaussian}),
$\phi(\epsilon d)=(\epsilon d)^2\log(\epsilon d)$ (\emph{thin plate spline}), see
more examples in~\cite{Gut01,Bem19}. 

In accordance with~\eqref{eq:surrogate-condition}, we impose the following preference conditions
\begin{equation}
    \ba{ll}
    \hat f(x_{i(h)})\leq\hat f(x_{j(h)})-\sigma+\varepsilon_h,& \forall h=1,\ldots,M\ \mbox{such that}\ \pref(x_{i(h)},x_{j(h)})=-1\\
    \hat f(x_{i(h)})\geq\hat f(x_{j(h)})+\sigma-\varepsilon_h,& \forall h=1,\ldots,M\ \mbox{such that}\ \pref(x_{i(h)},x_{j(h)})=1\\
    |\hat f(x_{i(h)})-\hat f(x_{j(h)})|\leq \sigma+\varepsilon_h,& \forall h=1,\ldots,M\ \mbox{such that}\ \pref(x_{i(h)},x_{j(h)})=0
    \ea
\label{eq:RBF-pref}
\end{equation}
where $\sigma>0$ is a given tolerance
and $\varepsilon_h$ are slack variables, $\varepsilon_h\geq 0$,
$h=1,\ldots,M$.

Accordingly, the coefficient vector $\beta=[\beta_1\ \ldots\ \beta_N]'$  is obtained by 
solving the following convex optimization problem
\begin{equation}
    \ba{rll}
    \min_{\beta,\varepsilon} &\displaystyle{\sum_{h=1}^{M} c_h\varepsilon_h+\frac{\lambda}{2}\sum_{k=1}^N\beta_{k}^2 }\\
    \st & \displaystyle{\sum_{k=1}^N(\phi(\epsilon d(x_{i(h)},x_{k})-\phi(\epsilon d(x_{j(h)},x_{k}))\beta_k\leq 
    -\sigma+\varepsilon_h},& \forall h:\ b_h=-1\\
     &\displaystyle{\sum_{k=1}^N(\phi(\epsilon d(x_{i(h)},x_{k})-\phi(\epsilon d(x_{j(h)},x_{k}))\beta_k\geq 
    \sigma-\varepsilon_h},& \forall h:\ b_h=1\\
    &\displaystyle{ \sum_{k=1}^N(\phi(\epsilon d(x_{j(h)},x_{k})-\phi(\epsilon d(x_{i(h)},x_{k}))\beta_k\leq\sigma+\varepsilon_h},& \forall h:\ b_h=0\\
    & \displaystyle{\sum_{k=1}^N(\phi(\epsilon d(x_{i(h)},x_{k})-\phi(\epsilon d(x_{j(h)},x_{k}))\beta_k\geq-\sigma-\varepsilon_h},& \forall h:\ b_h=0\\  & h=1,\ldots,M
    \ea
\label{eq:QP-pref}
\end{equation}
where $c_h$ are positive weights, for example $c_h=1$, $\forall h=1,\ldots,M$.
The scalar $\lambda$ is a regularization parameter. When $\lambda>0$ problem~\eqref{eq:QP-pref} is a quadratic programming (QP) problem that, since 
$c_h>0$ for all $h=1,\ldots,M$, admits a unique solution. If $\lambda=0$ problem~\eqref{eq:QP-pref} becomes a
linear program (LP), whose solution may not be unique.

Note that the use of slack variables $\varepsilon_h$ in~\eqref{eq:QP-pref}
allows one to relax the constraints imposed by the specified preference vector $B$.
Constraint infeasibility might be due to an inappropriate selection
of the RBF and/or to outliers in the acquired components $b_h$ of vector $B$.
The latter condition may easily happen when preferences $b_h$ are
expressed by a human decision maker in an inconsistent way.

For a given set $X=\{x_1\ \ldots\ x_N\}$ of samples, setting up~\eqref{eq:QP-pref} requires 
computing the $N\times N$ symmetric matrix $\Psi$ whose $(i,j)$-entry is
\begin{equation}
    \Psi_{ij}=\phi(\epsilon d(x_i,x_j))
\label{eq:RBF-matrix}
\end{equation}
with $\Psi_{ii}=1$ for the inverse quadratic and Gaussian
RBF, while for the thin plate spline RBF $\Psi_{ii}=\lim_{d\rightarrow 0}\phi(\epsilon d)=0$.
Note that if a new sample $x_{N+1}$ is collected, updating matrix $\Psi$ only requires computing $\phi(d(x_{N+1},x_j),\epsilon)$ for all $j=1,\ldots,N+1$.

\begin{figure}[t]
\begin{center}{\includegraphics[width=\hsize]{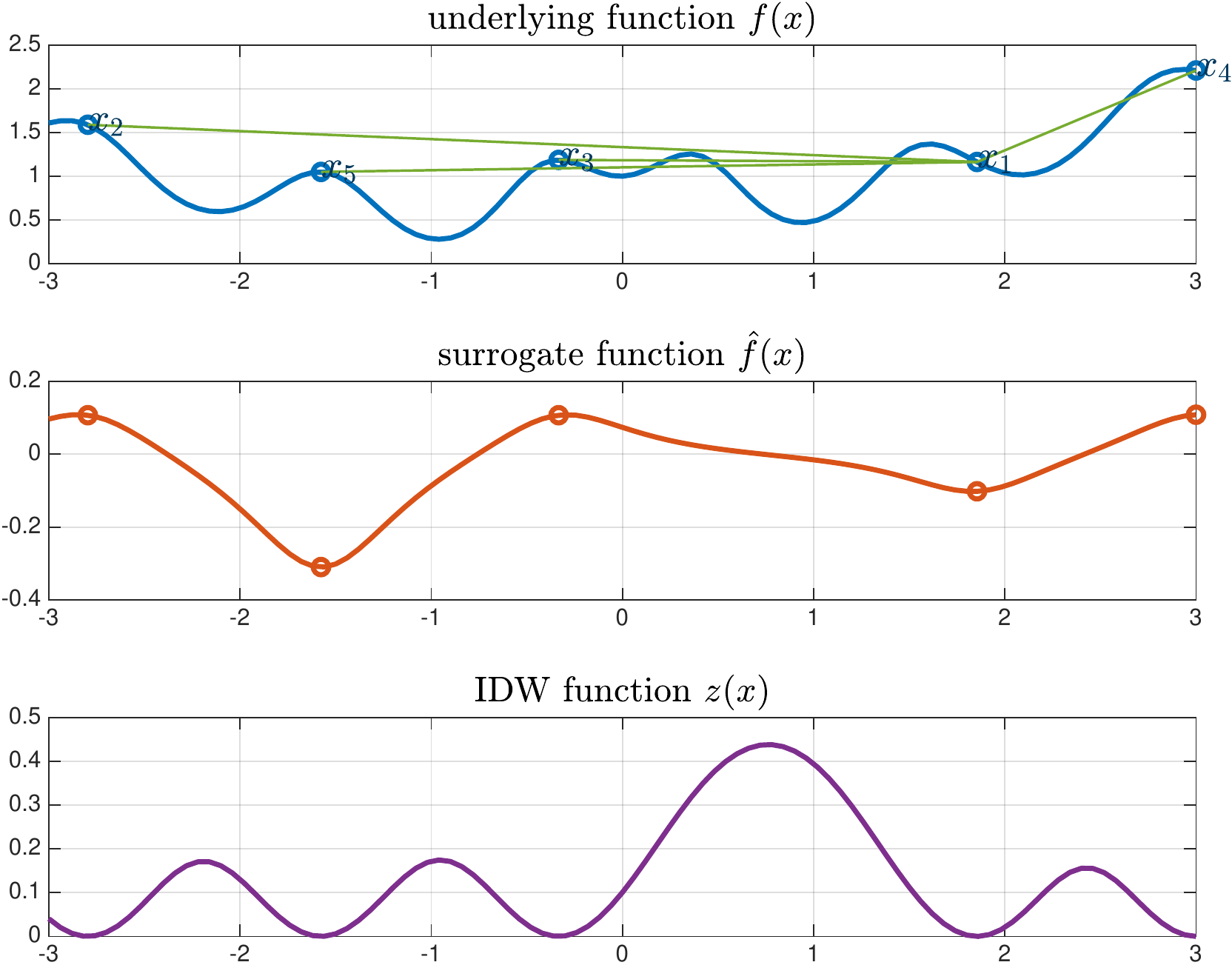}}\end{center}
\caption{Example of surrogate function $\hat f$ (middle plot) 
based on preferences resulting from function $f$ (top plot, blue) 
as in~\eqref{eq:f-1d-example}. Pairs of samples generating comparisons
are connected by a green line. IDW exploration function $z$ (bottom plot)}
\label{fig:surrogate}
\end{figure}

An example of surrogate function $\hat f$ constructed based on preferences
generated as in~\eqref{eq:pref_fun-f} by the following scalar function~\cite{Bem19}
\begin{equation}
    f(x)=\left(1+\frac{x\sin(2x)\cos(3x)}{1+x^2}\right)^2+\frac{x^2}{12}+\frac{x}{10}
\label{eq:f-1d-example}
\end{equation}
is depicted in Figure~\ref{fig:surrogate}. The surrogate is generated from $N=6$ samples
by solving the LP~\eqref{eq:QP-pref} ($\lambda=0$) with matrix $\Phi$ generated by the inverse quadratic 
RBF with $\epsilon=2$ and $\sigma=\frac{1}{N}$.

\subsection{Self-calibration of RBF} \label{Self-cal}
Computing the   surrogate $\hat f$  requires to choose the hyper-parameter $\epsilon$ defining the shape of the RBF $\phi$ (Eq.~\eqref{eq:rbf}). This parameter can be tuned through $K$-fold cross-validation~\cite{stone1974cross}, by splitting the $M$ available pairwise comparisons into $K$ (nearly equally sized) disjoint subsets.
To this end, let us define the index sets 
 $\Ss{i}$, $i=1,\ldots,K$, such that $  \cup_{i=1}^K\Ss{i}=\{1,\ldots,M\}$, $\Ss{i} \cap \Ss{j} = \emptyset$, for all  $i,j=1,\ldots,K$, $i\neq j$. For a given $\epsilon$ and for all $i=1,\ldots,K$, the preferences indexed by the  set $\{1,\ldots,M\} \setminus \Ss{i}$  are   used to fit   the surrogate function $\hat f_\epsilon$ by solving~\eqref{eq:QP-pref}, while the  performance of $\hat f_\epsilon$ in predicting comparisons indexed by $\Ss{i}$  is quantified in terms of number of  correctly  classified   preferences $\mathcal{C}_i(\epsilon)=\sum_{h \in \Ss{i}} \eta_h(\epsilon)$, 
 where $\eta_h(\epsilon)=1$ if $  \pref(x_{i(h)},x_{j(h)}) = \hat\pref_\epsilon(x_{i(h)},x_{j(h)})$ or $0$ otherwise, and $\hat\pref_\epsilon$  is the preference function induced by  $\hat f_\epsilon$ as in~\eqref{eq:pref_fun-f}. Since the hyper-parameter $\epsilon$ is scalar, a fine grid search can be used to find the value of $\epsilon$  maximizing $\sum_{i=1}^K\mathcal{C}_i(\epsilon)$.

 Since in active preference learning the number $M$ of observed pairwise preferences is usually  small, we use $\Ss{h}=\{h\}$, $h=1,\ldots,M$,  namely 
 $M$-fold cross validation or \emph{leave-one-out},  to better exploit the $M$ available  comparisons.

Let $x^\star_N\in\rr^n$ be the best vector of decision variables in the finite set 
$X=\{x_1,\ldots,x_N\}$, that is
\begin{equation}
    \pref(x^\star_N,x)\leq 0,\ \forall x\in X
\label{eq:pref-x-star}
\end{equation}
Since in active preference learning one is mostly interested in correctly predicting the preference w.r.t. the best optimal point $x^\star_N$, the solution of problem~\eqref{eq:QP-pref} and the corresponding  score $\mathcal{C}_i(\epsilon)$ are not computed for  all 
 indexes $h$ such that $x_{i(h)}=x^\star_N$, that is   the preferences involving $x^\star_N$ are only used for training and not for testing. 

The $K$-fold cross-validation  procedure for self-calibration  requires to formulate and solve problem~\eqref{eq:QP-pref} $K$ times ($M=N-1$ times in case of leave-one-out cross validation, or less when comparisons involving $x_N^\star$ are only used for training).   In order to reduce computations, self-calibration can be executed only at a subset $\mathcal{I}_{\rm{sc}} \subseteq \{1,\ldots,N_{\rm max}-1\}$ of iterations.
 
\section{Acquisition function}
\label{sec:acquisition}
Let $x^\star_N\in\rr^n$ be the best vector of decision variables defined in~\eqref{eq:pref-x-star}.
Consider the following procedure: ($i$) generate a new sample by pure minimization of the surrogate function $\hat f$ 
defined in~\eqref{eq:rbf},
\[
    x_{N+1}=\arg\min\hat f(x)\ \st\ \ell\leq x\leq u,\ x\in\XX 
\]
with $\beta$ obtained by solving the LP~\eqref{eq:QP-pref}, ($ii$) evaluate
$\pref(x_{N+1},x^\star_N)$, ($iii$) update $\hat f$, and ($iv$) iterate over $N$.
Such a procedure may easily miss the global minimum of~\eqref{eq:glob-opt-pref}, a phenomenon 
that is well known in global optimization based on surrogate functions:
purely minimizing the surrogate function may lead to converge to a point that is not the global minimum of the original function~\cite{Jon01,Bem19}. Therefore, the \emph{exploitation} of the surrogate function $\hat f$
is not enough to look for a new sample $x_{N+1}$, but also an \emph{exploration} objective
must be taken into account to probe other areas of the feasible space.

In the next paragraphs we propose two different acquisition functions that can be used to define the new sample $x_{N+1}$ to compare the current best sample $x^\star_N$ to.

\subsection{Acquisition based on inverse distance weighting}
Following the approach suggested in~\cite{Bem19}, we construct an exploration function using ideas
from inverse distance weighting (IDW). Consider the \emph{IDW exploration function}
$z:\rr^n\to\rr$ defined by
\begin{equation}
    z(x)=\left\{\ba{ll}
0 & \mbox{if}\ x\in\{x_1,\ldots,x_N\}\\
\tan^{-1}\left(\frac{1}{\sum_{i=1}^Nw_i(x)}\right)&
\mbox{otherwise}\ea\right.
\label{eq:IDW-distance}
\end{equation}
where $w_i:\rr^n\to\rr$ is defined by~\cite{She68}
\begin{equation}
    w_i(x)=\frac{1}{d^2(x,x_i)}
\label{eq:w-IDW-basic}
\end{equation}
Clearly $z(x_i)=0$ for all $x_i\in X$, and $z(x)>0$ in $\rr^n\setminus X$. The arc tangent
function in~\eqref{eq:IDW-distance} avoids that $z(x)$ gets excessively large
far away from all sampled points. Figure~\ref{fig:surrogate} shows the IDW exploration function $z$ obtained
from~\eqref{eq:IDW-distance} for the example generated from~\eqref{eq:f-1d-example}.

Given an exploration parameter $\delta\geq 0$, the \emph{acquisition function} $a:\rr^n\to\rr$
is defined as
\begin{equation}
    a(x)=\frac{\hat f(x)}{\Delta\hat F}-\delta z(x)
\label{eq:acquisition}
\end{equation}
where 
\[
    \Delta \hat F=\max_i\{\hat f(x_i)\}-\min_i\{\hat f(x_i)\}
\]
is the range of the surrogate function on the samples in $X$. By setting
\begin{subequations}
\begin{equation}
    y=M\beta
\end{equation}
we get $\hat f(x_i)=y_i$, $\forall i=1,\ldots,N$, and therefore
\begin{equation}
    \Delta\hat F = \max(y)-\min(y)
\end{equation}
\label{eq:Delta F}%
\end{subequations}
Clearly $\Delta \hat F \geq \sigma$ if at least one comparison $b_h=\pref(x_{i(h)},x_{i(h)}) \neq 0$.  The scaling factor $\Delta \hat F$ is used to simplify the choice of the exploration
parameter $\delta$.

The following lemma immediately derives from~\cite[Lemma~2]{Bem19}:
\begin{lemma}
Function $a$ is differentiable everywhere on $\rr^n$.
\label{lemma:sz}
\end{lemma}

As we will detail below, given a set $X$ of $N$ samples $\{x_1,\ldots,x_N\}$ and a vector $B$
of preferences defined by~\eqref{eq:pref-vector}, the next sample $x_{N+1}$ 
is defined by solving the global optimization problem
\begin{equation}
    x_{N+1}=\arg\min_{\ell\leq x\leq u,\ x\in\XX
                     } a(x)
\label{eq:xNp1}
\end{equation}
Problem~\eqref{eq:xNp1} can be solved very efficiently using various global optimization techniques, either derivative-free~\cite{RS13} or, if $\XX=\{x: g(x)\leq 0\}$ and $g$ is also differentiable, derivative-based. In case some components of vector $x$ are restricted to be integer,~\eqref{eq:xNp1} can be solved by mixed-integer programming.

\subsection{Acquisition based on maximum likelihood of improvement}
We show how the surrogate function $\hat f$ derived by solving  problem~\eqref{eq:QP-pref} can be seen as a maximum likelihood estimate of 
an appropriate probabilistic model. The analyses described in the following are inspired by the probabilistic interpretation of \emph{support vector machines} described in~\cite{franc2011}.

Let $\lambda>0$ and let $\Phi(\epsilon,X,x_{i(h)},x_{j(h)})$ be the $N$-dimensional vector obtained by collecting
the terms $\phi(\epsilon d(  x_{i(h)},x_k))-\phi(\epsilon d( x_{j(h)},x_k))$, with
$h=1,\ldots,M$, $k=1,\ldots,N$.

Let us rewrite the QP problem~\eqref{eq:QP-pref} without the slack variables $\varepsilon_i$ as 
\begin{equation} \label{eqn:solveP}
 \ba{rll}
{\displaystyle\min_{\beta}} &{\displaystyle\sum_{h=1}^{M}} c_{h} \ell_{b_h}(\Phi(\epsilon,X,x_{i(h)},x_{j(h)})'\beta)+ \frac{\lambda}{2} \left\| \beta \right\|^2
\ea 
\end{equation}
where 
\begin{subequations}
\begin{align}
\ell_{-1}(\Phi(\epsilon,X,x_{i(h)},x_{j(h)})'\beta) = & \max\{0, \Phi(\epsilon,X,x_{i(h)},x_{j(h)})'\beta+ \sigma) \}  \\   
\ell_{1}(\Phi(\epsilon,X,x_{i(h)},x_{j(h)})'\beta) = & \max\{0;-\Phi(\epsilon,X,x_{i(h)},x_{j(h)})'\beta + \sigma \}  \\
\ell_{0}( \Phi(\epsilon,X,x_{i(h)},x_{j(h)})'\beta) = & \max\{0, \pm\Phi(\epsilon,X,x_{i(h)},x_{j(h)})'\beta- \sigma\} 
\end{align} 
\end{subequations}
are piecewise linear convex functions of $\beta$, for all $h=1,\ldots,M$.

\begin{theorem}\label{Th:eqprob}
	For a given hyper-parameter $\lambda>0$, let $\beta(\lambda)$ be the minimizer of problem~\eqref{eqn:solveP} and let $\tau(\lambda) = \left\| \beta(\lambda)\right\|$. Then vector $u^\star=\frac{\beta(\lambda)}{\tau(\lambda)}$
is the minimizer of the following
problem
	\begin{align}
 \min_{u: \|u\|=1} \sum_{h=1}^{M} c_{h} \ell_{b_h}(\tau(\lambda)\Phi(\epsilon,X,x_{i(h)},x_{j(h)})'u)  \label{eqn:utau}
	\end{align}
\end{theorem}
\begin{proof}
	\ See Appendix. \\
\end{proof}	

In order to avoid heavy notation, we restrict the coefficients $c_h$ in~\eqref{eq:QP-pref} such that they are equal when the preference $b_h$ is the same, that is $c_h=\bar c_{b_h}$ where $\bar c_{-1},\bar c_0,\bar c_{1}$ are given positive weights.

Let us now focus on problem~\eqref{eqn:utau} and consider the joint p.d.f.
\begin{align} \label{eqn:distP}
p(\Phi,t;\bar c,\tau,u) = Z(\bar c,\tau,u)e^{-\bar c_{t}\ell_{t}(\tau \Phi'u)} \kappa(\Phi), 
\end{align} 
defined for $\Phi\in\mathbb{R}^N$ and $t\in\{-1,0,1\}$, and parametrized by $\bar c=[\bar c_{-1}\ \bar c_0\ \bar c_1]'$, a strictly positive scalar $\tau$, and a generic unit vector $u$. 

The distribution~\eqref{eqn:distP} is composed by three terms. The first term $Z(\bar c,\tau,u)$ is a normalization constant. We will show next that $Z(\bar c,\tau,u)$
does not depend on $u$ when we restrict $\|u\|=1$. The second term $e^{-\bar c_{t}\ell_{t}(\tau \Phi'u)}$ depends on all the parameters $(\bar c,\tau,u)$ and it is related to the objective function minimized in~\eqref{eqn:utau}.   The last term $\kappa(\Phi)$ ensures integrability of $p(\Phi,t;\bar c,\tau,u)$  and that the normalization constant $Z$ does not depend on $u$, as discussed next. A possible choice 
for $\kappa$ is 
 $\kappa(\Phi)=e^{-\Phi'\Phi}$.

The normalization constant $Z$ is given by
\begin{align}
Z(\bar c,\tau,u) = \frac{1}{\sum_{t\in\{-1,0,1\}}I_t(\bar c_t,\tau,u)}
\end{align}
where for $t\in\{-1,0,1\}$ the term $I_t(\bar c_t,\tau,u)$ is the integral defined as 
\begin{align} \label{eqn:inty}
I_t(\bar c_t,\tau,u) = \int_{ \Phi \in \mathbb{R}^N} e^{-\bar c_{t}\ell_{t}(\tau \Phi'u)}\kappa(\Phi)  d\Phi
\end{align}
The following Theorem shows that $I_t(\bar c_t,\tau,u)$ does not depend on $u$, and so $Z(\bar c,\tau,u)$ is also independent of $u$.
\begin{theorem} \label{Th:int}
	Let $\kappa(\Phi)$ in~\eqref{eqn:distP} be $\kappa(\Phi)=e^{-\Phi'\Phi}$.  For any $t \in\{-1, 0, 1\}$, 
	\begin{align}
	I_t(\bar c_t,\tau,u) = I_t(\bar c_t,\tau,\bar u) \ \ \forall u,\bar u: \|u\|=\|\bar u\|=1.  
	\end{align}
\end{theorem}
\begin{proof}\  See Appendix.\\
\end{proof}
Because of Theorem~\ref{Th:int}, since now on, when
we restrict $\|u\|=1$, we will drop the dependence on $u$ of $Z(\bar c,\tau,u)$ and simply write $Z(\bar c,\tau)$.

Let us assume that the samples of the training sequence 
$\{\Phi(\epsilon,X,h),b_h\}_{h=1}^M$ are i.i.d.\ and generated
from the joint distribution $p(\Phi,t;\bar c,\tau,u)$ defined in~\eqref{eqn:distP}.
The negative log of the probability of the dataset   $\{\Phi(\epsilon,X,x_{i(h)},x_{j(h)}),b_h\}_{h=1}^M$ given $\bar c,\tau,u$ is
\begin{align}
L(\bar c,\tau,u) = & -\sum_{h=1}^{M}\log p(\Phi(\epsilon,X,x_{i(h)},x_{j(h)}),b_h;\bar c,\tau,u) = \nonumber\\
= & -M\log Z(\bar c,\tau) - \sum_{h=1}^{M} \log \kappa(\Phi(\epsilon,X,x_{i(h)},x_{j(h)}))\nonumber\\& +  \sum_{h=1}^{M} \bar c_{b_h}\ell_{b_h}(\tau \Phi(\epsilon,X,x_{i(h)},x_{j(h)})'u)
\end{align} 
Thus, for fixed values of $\bar c$ and $\tau=\|\beta(\lambda)\|$,
by Theorem~\ref{Th:eqprob}  the minimizer $u^\star_L(\lambda)$
of
\[
    \min_{u:\ \|u\|=1}L(\bar c,\tau(\lambda),u)
\]
is $u^\star_L=\frac{\beta(\lambda)}{\tau(\lambda)}$.
In other words, for any fixed $\lambda>0$, the solution $\beta(\lambda)$
of the QP problem~\eqref{eq:QP-pref}
can be reinterpreted as $\tau$ times the maximizer $u^\star_L(\lambda)$ of the joint likelihood $L(\bar c,\tau,u)$   with respect to $u$, $\|u\|=1$,
when $\tau=\|\beta(\lambda)\|$.

It is interesting to note that the marginal p.d.f. derived from the probabilistic model~\eqref{eqn:distP} is equal to
\begin{align}
p(\Phi;\bar c,\tau,u) = \sum_{t=-1,0,1} p(\Phi,t;\bar c,\tau,u)=
Z(\bar c,\tau)\kappa(\Phi)\sum_{t=-1,0,1}e^{-\bar c_{t}\ell_{t}(\tau \Phi'u)}
\end{align}
and therefore the corresponding preference posterior probability is
\begin{align}
p(t|\Phi;\bar c,\tau,u) = \frac{p(\Phi,t;\bar c,\tau,u)}{p(\Phi;\bar c,\tau,u)} = \frac{  e^{-\bar c_{t}\ell_{t}(\Phi'\beta)}}{\displaystyle{\sum_{t=-1,0,1}e^{-\bar c_{t}\ell_{t}( \Phi'\beta)}}}
\label{eq:posterior}
\end{align}
where $\beta=\tau u$. 

The preference posterior probability given~\eqref{eq:posterior} can be 
used now to explore the vector space $\rr^n$, as we describe next.

Let $\beta$ be the vector obtained
by solving~\eqref{eq:QP-pref} with $N$ samples and $M$ preferences.
Let us treat again $x_{N+1}$ as a free sample $x$
to optimize and consider~\eqref{eq:posterior} also for the new
generic $(M+1)$th comparison
\[
    \Phi(\epsilon,X,x,x^\star(N))=\matrice{c}{\phi(\epsilon d(x,x_1))-\phi(\epsilon d(x^\star_N,x_1)\\
\ldots\\
\phi(\epsilon d(x,x_N))-\phi(\epsilon d(x^\star_N,x_N)
}
\]
A criterion to choose $x_{N+1}$ is to maximize the preference posterior probability  of obtaining a ``better'' sample compared to the current ``best'' sample $x^\star_N$ given by~\eqref{eq:posterior}, or equivalently of getting $\pref(x_{N+1},x^\star_N)=-1$.
This can be achieved by the following acquisition function
\begin{align}
    a(x)= & -p\left(t=-1\left|\Phi(\epsilon,X,x,x^\star_N);\bar c,\|\beta\|,\frac{\beta}{\|\beta\|}\right.\right) \nonumber \\
= & -\frac{  e^{-\bar c_{-1}\ell_{-1}(\Phi(\epsilon,X,x,x^\star_N)'\beta)}}{\displaystyle{\sum_{t=-1,0,1}e^{-\bar c_{t}\ell_{t}( \Phi(\epsilon,X,x,x^\star_N)'\beta)}}}
    \label{eq:acquisition-improvement-prob}
\end{align}

\begin{figure}[t]
\begin{center}{\includegraphics[width=\hsize]{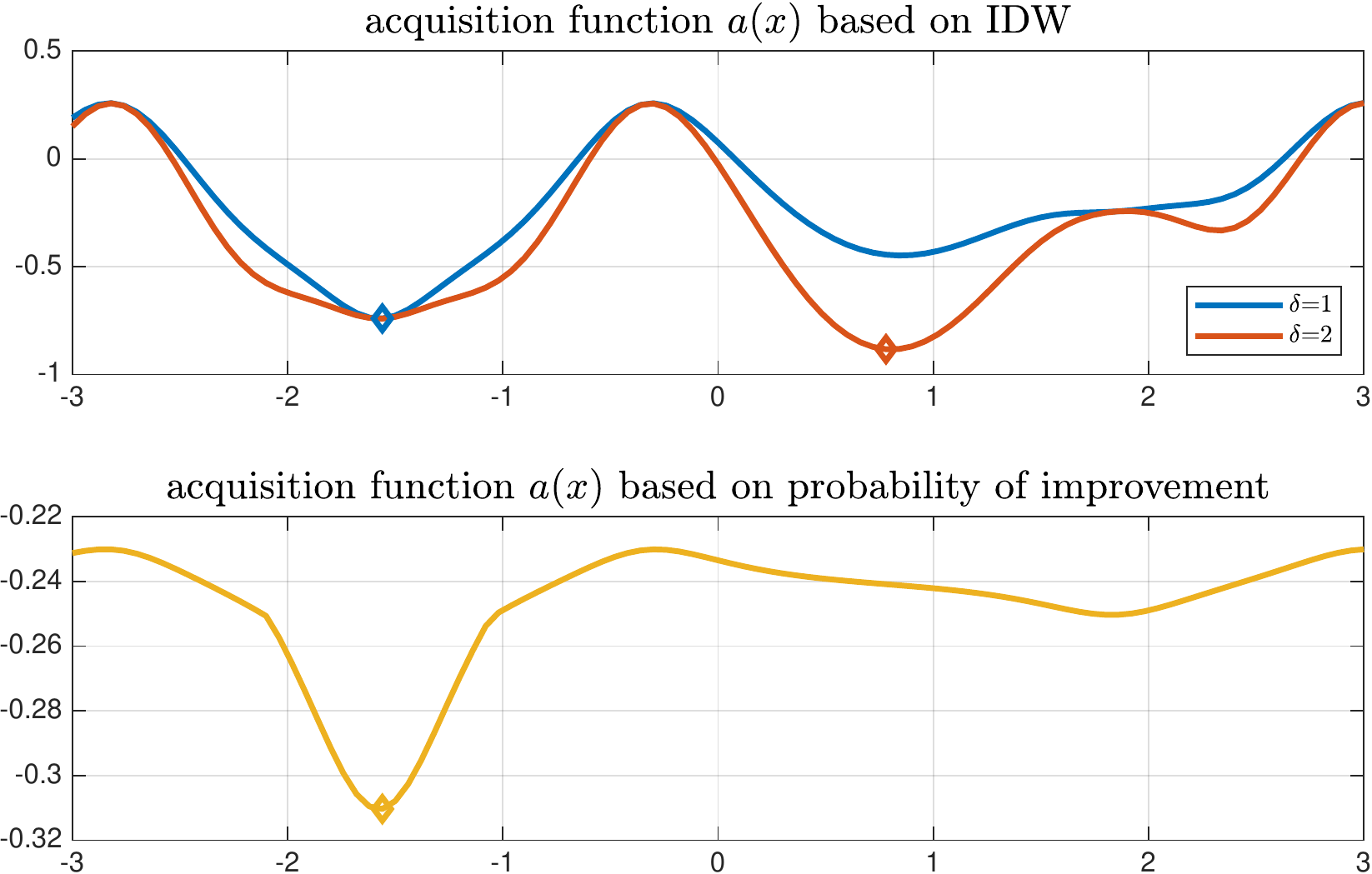}}\end{center}
\caption{Example of acquisition functions $a$  
based on preferences resulting from function $f$ as 
in~\eqref{eq:f-1d-example} and Figure~\ref{fig:surrogate}. 
IDW-based acquisition function $a$ as in~\eqref{eq:acquisition} with $\delta=1$ and $\delta=2$ (top plot), acquisition function $a$ as in~\eqref{eq:acquisition-improvement-prob} based in probability of improvement.
The minimum of $a$ is highlighted with a diamond.}
\label{fig:acquisition}
\end{figure}

Examples of acquisition functions $a$ constructed based on preferences
generated by the function $f$ defined in~\eqref{eq:f-1d-example}
are depicted in Figure~\ref{fig:acquisition}, based on the same
setting as in Figure~\ref{fig:surrogate}.

\subsection{Scaling}
\label{sec:scaling}
Different components $x^j$ of $x$ may have different upper and lower bounds
$u^j$, $\ell^j$. Rather than 
using weighted distances as in stochastic process model approaches
such as Kriging methods~\cite{SWMW89,JSW98}, we simply rescale the variables in optimization problem~\eqref{eq:glob-opt-pref}
to range in $[-1,1]$. As described in~\cite{Bem19}, we first tighten the given range
$B_{\ell,u}=\{x\in\rr^n: \ell\leq x\leq u\}$ by computing the bounding box $B_{\ell_s,u_s}$ of the set $\{x\in\rr^n:\ x\in\XX 
\}$ and replacing $B_{\ell,u}$ with $B_{\ell_s,u_s}$.
The bounding box $B_{\ell_s,u_s}$ is obtained by solving the following $2n$
optimization problems
\begin{equation}
    \ba{rcl}
    \ell_s^i&=&\min_{\ell\leq x\leq u,\ x\in\XX
                } e_i'x\\
    u_s^i&=&\max_{\ell\leq x\leq u,\ x\in\XX
                } e_i'x
    \ea 
\label{eq:BBox}
\end{equation}
where $e_i$ is the $i$th column of the identity matrix, $i=1,\ldots,n$. Note that
Problem~\eqref{eq:BBox} is a linear programming (LP) problem in
case of linear inequality constraints $\XX=\{x\in\rr^n\ Ax\leq b\}$
. Then, we operate with new scaled variables $\bar x\in\rr^n$, $\bar x_i\in[-1,1]$,
and replace the original preference learning problem~\eqref{eq:glob-opt-pref} with
\begin{equation}
    \mbox{find}\ \bar x^\star\ \mbox{such that}\ \pref(x(\bar x^\star),x(\bar x))\leq 0,\ \forall \bar x\in\XX_s,\ \ell_s\leq x(\bar x)\leq u_s
\label{eq:glob-opt-pref-scaled}
\end{equation}
where the scaling mapping $x:\rr^n\to\rr^n$ is defined as
\begin{equation}
    x^j(\bar x)=\frac{u_s^j-\ell_s^j}{2}\bar x^j+\frac{u_s^j+\ell_s^j}{2},\ j=1,\ldots,n
\label{eq:scaled-pref}
\end{equation}
where clearly $x^j(-1)=-1$, $x^j(1)=1$,
and $\XX_s$ 
is the set 
\begin{equation}
    \XX_s=\{\bar x\in\rr^n:\ x(\bar x)\in\XX\} 
\label{eq:scaled_X}
\end{equation}
When $\XX$ is the polyhedron $\{x:\ Ax\leq b\}$,~\eqref{eq:scaled_X} 
corresponds to defining the new polyhedron
\begin{equation}
    \XX_s=\{\bar x:\ \bar A\bar x\leq \bar b\}
    \label{eq:gs_Ab_bar}
\end{equation}
where
\begin{equation}
    \ba{rcl}
       \bar A&=&A\diag(u_s-\ell_s)\\
        \bar b&=&b-A(\frac{u_s+\ell_s}{2})
    \ea 
\label{eq:Ab_bar}
\end{equation}
and $\diag(u_s-\ell_s)$ is the diagonal matrix whose diagonal elements
are the components of $u_s-\ell_s$. 

Note that in case the preference function $\pref$ is related 
to an underlying function $f$ as in~\eqref{eq:pref_fun-f}, 
applying scaling is equivalent to formulate the following
scaled preference function
\begin{equation}
    \pref(\bar x_1,\bar x_2)=\left\{\ba{ll}
    -1 &\mbox{if $f(x(\bar x_1))<f(x(\bar x_2))$}\\
    0 &\mbox{if $f(x(\bar x_1))=f(x(\bar x_2))$}\\
    1 &\mbox{if $f(x(\bar x_1))>f(x(\bar x_2))$}
    \ea\right.
\label{eq:pref_fun-f-scaled}
\end{equation}

\section{Preference learning algorithm}
\label{sec:algorithm}

Algorithm~\ref{algo:idwgopt-pref} summarizes the proposed approach to solve 
the optimization problem~\eqref{eq:glob-opt-pref} by preferences
using RBF interpolants~\eqref{eq:rbf}  and the acquisition functions defined in Section~\ref{sec:acquisition}.

\begin{algorithm}[ht!]
	\caption{Preference learning algorithm based on RBF+IDW acquisition function}
	\label{algo:idwgopt-pref}
	~~\textbf{Input}: Upper and lower bounds $(\ell,u)$, constraint set $\XX$;
	number $N_{\rm init}\geq 2$ of initial samples, number $N_{\rm max}\geq N_{\rm init}$ of maximum number of function 
	evaluations; $\delta\geq 0$; $\sigma>0$; $\epsilon > 0$; 
	self-calibration index set  $\mathcal{I}_{\rm{sc}} \subseteq \{1,\ldots,N_{\rm max}-1\}$.
	\vspace*{.1cm}\hrule\vspace*{.1cm}
	\begin{enumerate}[label*=\arabic*., ref=\theenumi{}]
		\item Tighten $(\ell,u)$ to $(\ell_s,u_s)$ as in~\eqref{eq:BBox};
		\item Scale problem as described in~Section~\ref{sec:scaling};
		\item \label{algo:latin} Generate $N_{\rm init}$ random samples $X=\{x_1,\ldots,x_{N_{\rm init}}\}$ using Latin hypercube sampling~\cite{MBC79};
				
		\item  $N\leftarrow 1$,  $i^\star\leftarrow 1$;
		\item \textbf{While} $N< N_{\rm max}$ \textbf{do}
		\begin{enumerate}[label=\theenumi{}.\arabic*., ref=\theenumi{}.\arabic*]
			\item \textbf{if} $N \geq N_{\rm init}$ \textbf{then}
					\begin{enumerate}[label=\theenumii{}.\arabic*., ref=\theenumii{}.\arabic*]
					\item \label{algo:AL1} \textbf{if} $N \in \mathcal{I}_{\rm sc}$ \textbf{then} recalibrate $\epsilon$ as described in Section~\ref{Self-cal};	
					\item  Solve~\eqref{eq:QP-pref} to define the surrogate function $\hat f$~\eqref{eq:rbf};
					\item \label{algo:acquisition} Define  acquisition function $a$ as in~\eqref{eq:acquisition} or~\eqref{eq:acquisition-improvement-prob};
					\item \label{algo:globopt} Solve global optimization problem~\eqref{eq:xNp1} and get $x_{N+1}$;	
					\end{enumerate}
			\item $i(N)\leftarrow i^\star$, $ j(N) \leftarrow  N+1$;
			\item Observe preference $b_N = \pref(x_{i(N)},x_{j(N)})$; 
			\item \textbf{if} $b_N=1$ \textbf{then set} $i^\star\leftarrow  j(N)$; 
			\item $N\leftarrow N+1$;
		\end{enumerate}
		\item \textbf{End}.
	\end{enumerate}
	\vspace*{.1cm}\hrule\vspace*{.1cm}
	~~\textbf{Output}: Global optimizer $x^\star=x_{i^\star}$.
\end{algorithm}

In Step~\ref{algo:latin} \emph{Latin Hypercube Sampling} (LHS)~\cite{MBC79} is used 
to generate the initial set $X$ of $N_{\rm init}$ samples. The generated samples may not satisfy the constraint
$x\in\XX$. We distinguish between two cases:
\begin{enumerate}
    \item [$i$)] the comparison $\pref(x_1,x_2)$ can be done even if $x_1\not\in\XX$ and/or $x_2\not\in\XX$;
    \item [$ii$)] $\pref(x_1,x_2)$ can only be evaluated if $x_1,x_2\in\XX$.
\end{enumerate}
In the first case, the initial comparisons are still useful to 
define the surrogate function. In the second case, a possible approach is to generate 
a number of samples larger than $N_{\rm init}$ and discard the samples $x_i\not\in\XX$.
An approach for performing this is suggested in~\cite[Algorithm~2]{Bem19}.

Step~\ref{algo:globopt} requires solving a global optimization problem. In this paper we 
use Particle Swarm Optimization (PSO)~\cite{Ken10,VV07} to solve problem~\eqref{eq:xNp1}.
Alternative global optimization methods such as DIRECT~\cite{Jon09} or others methods~\cite{HN99,RS13}
could be used to solve~\eqref{eq:xNp1}. Note that penalty functions can be used to
take inequality constraints~\eqref{eq:g(x)} into account, 
for example by replacing~\eqref{eq:xNp1} with 
\begin{equation}
    x_{N+1}=\arg\min_{\ell\leq x\leq u} a(x)+\rho\Delta \hat F\sum_{i=1}^q\max\{g_i(x),0\}^2
\label{eq:xNp1-penalty}
\end{equation}
where $\rho\gg 1$ in~\eqref{eq:xNp1-penalty}. 

Algorithm~\ref{algo:idwgopt-pref} consists of two phases: initialization and active learning. During initialization, sample  $x_{N+1}$ is simply retrieved from the initial set $X=\{x_1,\ldots,x_{N_{\rm init}}\}$. Instead, in the active learning phase, sample $x_{N+1}$ is obtained in Steps~\ref{algo:AL1}--\ref{algo:globopt} by solving the optimization problem~\eqref{eq:xNp1}.  Note that the construction of the acquisition function $a$ is rather heuristic,
therefore finding global solutions of very high accuracy of~\eqref{eq:xNp1}
is not required.

When using the acquisition function~\eqref{eq:acquisition},
the exploration parameter $\delta$ promotes sampling the space in $[\ell,u]\cap\XX$ in
areas that have not been explored yet. While setting $\delta\gg1$ makes
Algorithm~\ref{algo:idwgopt-pref} exploring the entire feasible region regardless
of the results of the comparisons, setting $\delta=0$ can make Algorithm~\ref{algo:idwgopt-pref} 
rely only on the surrogate function $\hat f$ and miss the global optimizer.
Note that using the acquisition function~\eqref{eq:acquisition-improvement-prob}
does not require specifying the hyper-parameter $\delta$. On the other hand,
the presence of the IDW function in the acquisition allows promoting an
exploration which is independent of the surrogate, and therefore $\delta$
might be a useful tuning knob to have. Clearly, the acquisition function~\eqref{eq:acquisition-improvement-prob} can be also augmented by the term $\delta z(x)$ as in~\eqref{eq:acquisition} to recover such exploration flexibility.

Figure~\ref{fig:surrogate} (upper plot) shows the samples generated by
Algorithm~\ref{algo:idwgopt-pref} when applied to minimize the function $f$~\eqref{eq:f-1d-example} in $[-3,3]$, by setting $\delta=1$, $N_{\rm max}=6$, $N_{\rm init}=3$, $\mathcal{I}_{\rm sc}=\emptyset$, 
$\Psi$ generated by the inverse quadratic RBF with $\epsilon=2$, and $\sigma=\frac{1}{N_{\rm max}}$.

\subsection{Computational complexity}
Algorithm~\ref{algo:idwgopt-pref} solves $N_{\rm max}-N_{\rm init}$
quadratic or linear programs~\eqref{eq:QP-pref} with growing size, namely
with $2N-1$ variables, a number $q$ of linear inequality constraints
with $N-1\leq q\leq 2(N-1)$ depending on the outcome of the preferences,
and $2$ equality constraints. Moreover, it solves
$N_{\rm max}-N_{\rm init}$ global optimization problems~\eqref{eq:xNp1} in
the $n$-dimensional space, whose complexity depends on the used global optimizer. 
The computation of matrix $\Psi$ requires overall
$N_{\rm max}(N_{\rm max}-1)$ RBF values $\phi(\epsilon d(x_i,x_j))$,
$i,j=1,\ldots,N_{\rm max}$, $j\neq i$. The leave-one-out cross validation executed at Step~\ref{algo:AL1} for recalibrating $\epsilon$ requires to formulate and solve problem~\eqref{eq:QP-pref} at most $N-1$ times.
On top of the above analysis, one has to take account the cost of evaluating
the preferences $\pref(x_{i(h)},x_{j(h)})$, $h=1,\ldots,N_{\rm max}-1$.

\section{Application to multi-objective optimization} 
\label{sec:multiobj}
The active preference learning methods introduced in the previous sections can be effectively used to solve multi-objective optimization problems of the form
\begin{subequations}
\begin{eqnarray}
    \min_z &&F(z)=\matrice{c}{F_1(z)\\ \vdots\\ F_{n}(z)}\label{eq:multi-obj-fun}\\[.5em]
    \st && g(z)\leq 0
\label{eq:multi-obj-constr}
\end{eqnarray}
\label{eq:multi-obj}%
\end{subequations}
where $z\in\rr^{n_z}$ is the optimization vector, $F_i:\rr^{n_z}\to\rr$, $i=1,\ldots,n$,
are the objective functions, $n\geq 2$, and $g:\rr^{n_z}\to\rr^{n_g}$ is the function defining
the constraints on $z$ (including possible box and linear constraints). In general Problem~\eqref{eq:multi-obj} admits infinitely many 
Pareto optimal solutions, leaving the selection of one of them a matter of \emph{preference}.

Pareto optimal solutions can be expressed by \emph{scalarizing} problem~\eqref{eq:multi-obj}
into the following standard optimization problem
\begin{subequations}
\begin{eqnarray}
    F^\star(x)=\min_z &&\sum_{i=1}^{n}x_iF_i(z)\label{eq:scal-multi-obj-fun}\\[.5em]
    \st && g(z)\leq 0
\end{eqnarray}
\label{eq:scal-multi-obj}%
\end{subequations}
where $x_1,\ldots,x_{n}$ are nonnegative scalar weights, and $F^\star:\rr^n\to\rr\cup\{-\infty\}$.
Let us model the preference between Pareto optimal solution through the preference function
$\pref:\rr^{n}\times\rr^n\to\{-1,0,1\}$ 
\begin{equation}
    \pref(x,y)=\left\{\ba{ll}
    -1 &\mbox{if $F^\star(x)$ is ``better'' than $F^\star(y)$}\\
    0 &\mbox{if $F^\star(x)$ is ``as good as'' than $F^\star(y)$}\\
    1 &\mbox{if $F^\star(y)$ is ``better'' than $F^\star(x)$}
    \ea\right.
\label{eq:mo-pref_fun}
\end{equation}
where $x,y\in\rr^n$. The optimal selection of a Pareto optimal solution can be therefore expressed 
as a preference optimization problem of the form~\eqref{eq:glob-opt-pref}, with $\ell=0$, $u=+\infty$,
$\XX=\rr^n$. 

Without loss of generality, we can set $\sum_{i=1}^{n}x_i=1$ and eliminate 
$x_{n}=1-\sum_{i=1}^{n-1}x_i$, so to solve a preference optimization problem with $n-1$
variables under the constraints  $x_i\geq 0$, $\sum_{i=1}^{n-1}x_i\leq 1$. 
In Section~\ref{sec:mo-results} we will illustrate the effectiveness of the active preference learning
algorithms introduced earlier in solving the multi-objective optimization problem~\eqref{eq:multi-obj}
under the preference function~\eqref{eq:mo-pref_fun}.

\section{Numerical results}
\label{sec:results}
In this section we test the active preference learning approach described in the previous sections 
on different optimization problems, only based on preference queries.

Computations are performed on an Intel i7-8550 CPU @1.8GHz machine in MATLAB R2019a.
Both Algorithm~\ref{algo:idwgopt-pref} and the Bayesian active preference learning algorithm
are run in interpreted code. Problem~\eqref{eq:xNp1} (or~\eqref{eq:xNp1-penalty}, in case of constraints) is solved by 
the PSO solver~\cite{VV09}. For judging the quality of 
the solution obtained by active preference learning, the best between the solution obtained
by running the optimization algorithm DIRECT~\cite{Jon09} through the NLopt interface~\cite{NLopt} 
and by running the PSO solver~\cite{VV09} was used as the reference global optimum.
The Latin hypercube sampling function~\texttt{lhsdesign} of the Statistics and Machine Learning Toolbox
of MATLAB is used to generate initial samples.

\subsection{Illustrative example}
We first illustrate the behavior of Algorithm~\ref{algo:idwgopt-pref} when solving 
the following constrained benchmark global optimization problem proposed by Sasena et al.~\cite{SPG02}:
\begin{equation}
    \ba{rl}
\min & \displaystyle{2+\frac{1}{100}(x_2-x_1^2)^2+(1-x_1)^2+2(2-x_2)^2+7\sin(\frac{1}{2}x_1)\sin(\frac{7}{10}x_1x_2)}
\\[1em]
\st &\displaystyle{-\sin(x_1-x_2-\frac{\pi}{8})\leq 0}\\[1em]
    &0\leq x_1,x_2\leq 5
    \ea    
\label{eq:sasena}
\end{equation}
The minimizer of problem~\eqref{eq:sasena} is $x^\star=[2.7450\ 2.3523]'$
with optimal cost $f^\star=-1.1743$. Algorithm~\ref{algo:idwgopt-pref} is run
with initial parameter $\epsilon=1$ and inverse quadratic RBF to fit 
the surrogate function, using the acquisition criterion~\eqref{eq:acquisition}
with $\delta=1$, $N_{\rm max}=25$, $N_{\rm init}=8$ feasible initial samples,
$\sigma=1$. Self-calibration is executed at steps $N$
indexed by $\mathcal{I}_{\rm{sc}}=\{8,12,17,21\}$ over a grid of $10$ values $\epsilon_\ell=\epsilon\theta_\ell$,
$\theta_\ell\in\Theta$, $\Theta=\{10^{-1+\frac{1}{5}(\ell-1)}\}_{\ell=1}^{10}$.

Figures~\ref{fig:sasena-1} shows the samples $X=\{x_1,\ldots,x_{N_{\rm max}}\}$
generated by a run of Algorithm~\ref{algo:idwgopt-pref}, 
Figure~\ref{fig:sasena-3} the best (unmeasured) value of the latent function $f$ as a function of the number of preference queries, Figure~\ref{fig:sasena-2} the shapes of $f$ and of the surrogate function $\hat f$. It is apparent that while $\hat f$ achieves the goal of driving the algorithm towards the global minimum, its shape is quite different from $f$, as it has been
constructed only to honor the preference constraints~\eqref{eq:RBF-pref} 
at sampled values. Therefore, given a new pair of samples $x_1$, $x_2$ 
that are located far away from the collected samples $X$, 
the surrogate function $\hat f$ may not be useful in predicting the outcome of the 
comparison $\pref(x_1,x_2)$.

It is apparent that $\hat f$ can be arbitrarily scaled and shifted without changing
the outcome of preferences. 
While the arbitrariness in scaling is taken into account by
the term $\Delta\hat F$ in~\eqref{eq:acquisition}, it would be
immediate to modify problem~\eqref{eq:QP-pref} to include the equality constraint
\begin{equation}
    \sum_{j=1}^N \phi(\epsilon d(x_{i^\star},x_{j}))\beta_j=0
    \label{eq:RBF-normalize}
\end{equation}
so that by construction $\hat f$ is zero at the current best sample $x_{i^\star}$.

\begin{figure}[t]
\centerline{\includegraphics[width=.8\hsize]{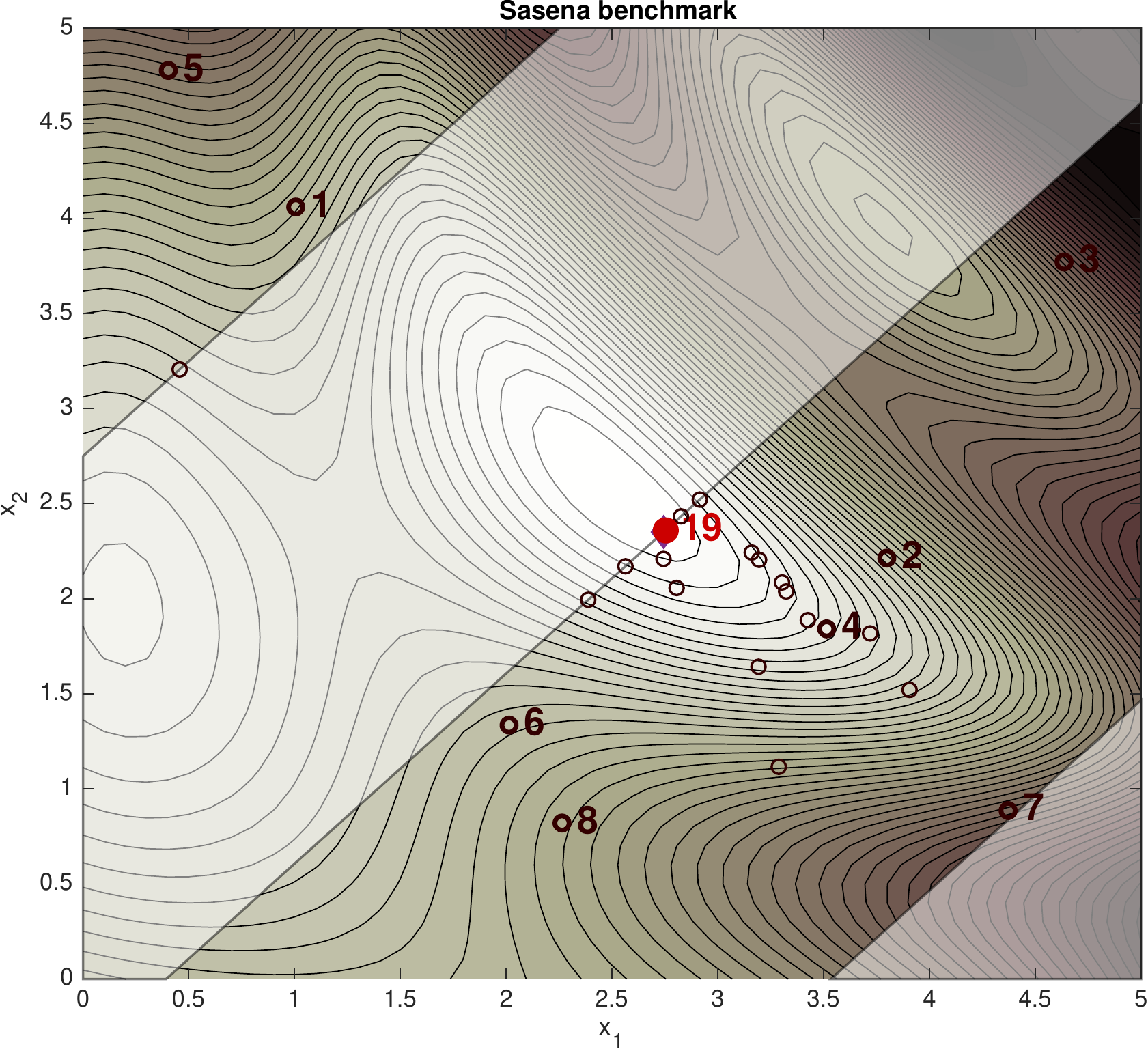}}
\caption{Level sets of the latent function $f$ and feasible domain
defined in~\eqref{eq:sasena}. The points $X$ generated by Algorithm~\ref{algo:idwgopt-pref}
accumulate towards the global constrained minimum as $N$ grows.}
\label{fig:sasena-1}
\end{figure}

\begin{figure}[t]
\centerline{\includegraphics[width=.8\hsize]{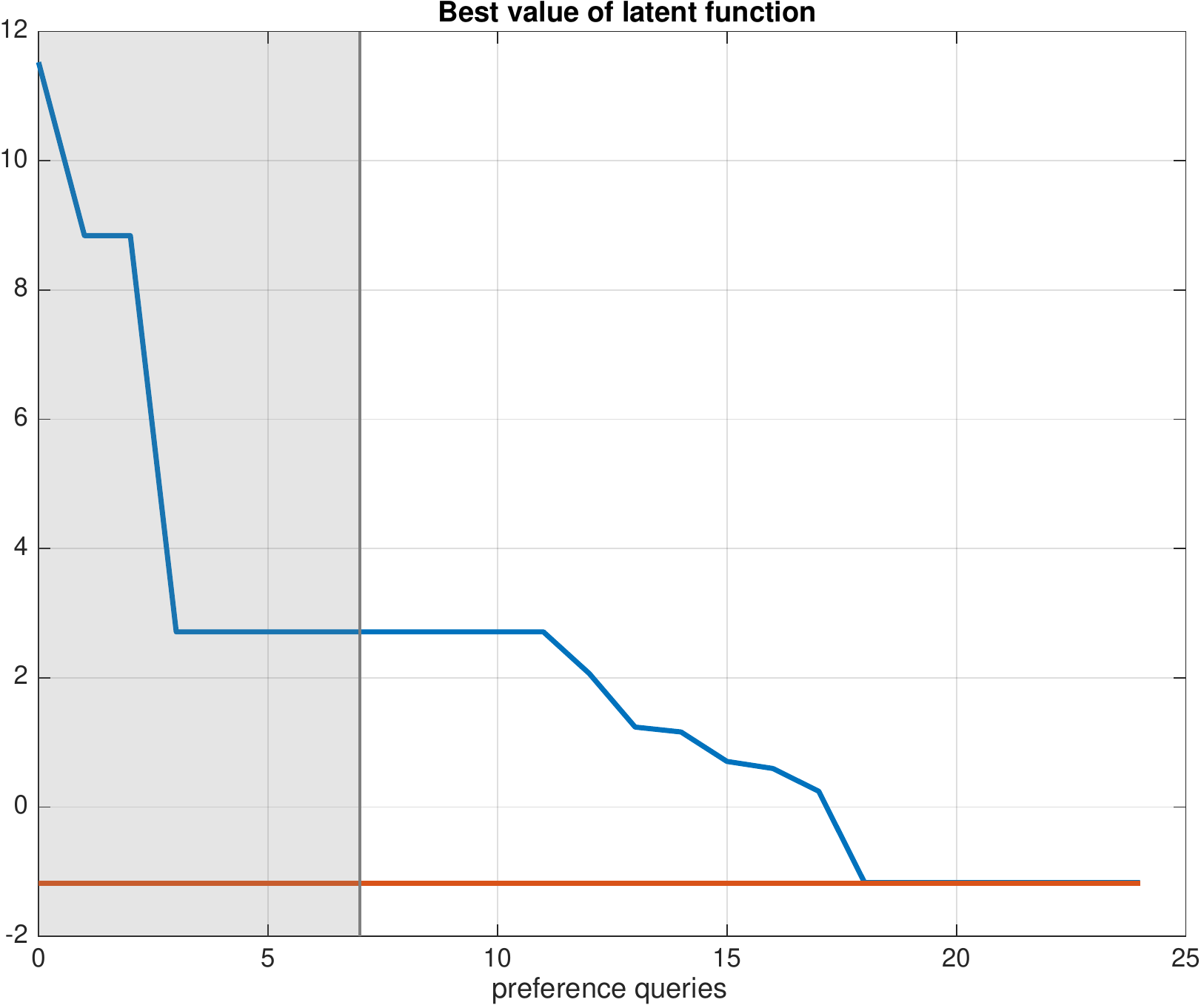}}
\caption{Best value of the latent function $f$ defined in~\eqref{eq:sasena}
as a function of the number of queried preferences. The vertical line 
denotes the last query after which active preference learning begins.}
\label{fig:sasena-3}
\end{figure}

\begin{figure}[t]
\centerline{\includegraphics[width=1.2\hsize]{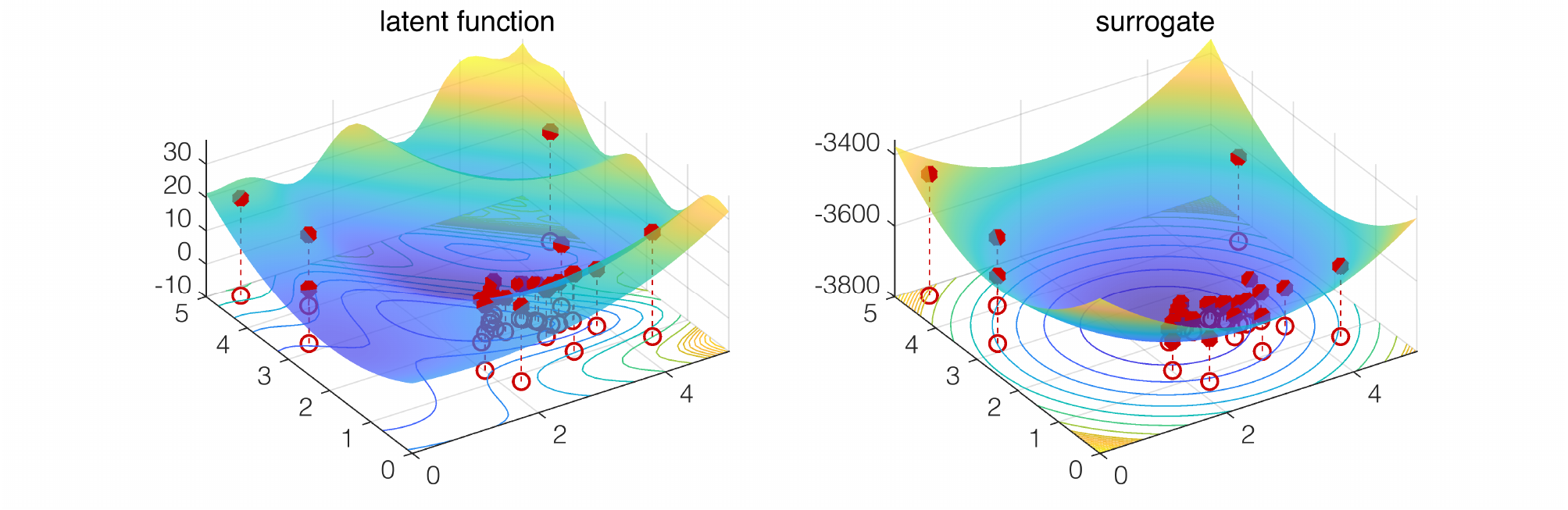}}
\caption{Latent function $f$ and surrogate $\hat f$ from the problem
defined in~\eqref{eq:sasena}, along with the samples $X$ (red circles) generated by
Algorithm~\ref{algo:idwgopt-pref}.}
\label{fig:sasena-2}
\end{figure}

\subsection{Benchmark global optimization problems}
\label{sec:benchmarks}
We test the proposed global optimization algorithm on standard benchmark global
optimization problems.
Problems \textsf{brochu-2d}, \textsf{brochu-4d}, \textsf{brochu-6d} were proposed  
in~\cite{BDG08} and are defined as follows:
\[
  \begin{array}{rcl}
f_d(x)&= & \displaystyle{\sum_{i=1}^d\sin(x_i)+\frac{1}{3}x_i+\sin(12 x_i)}\\[1em]
    f_{\textsf{brochu-2d}}(x)&= & -\max\{f_2(x)-1,0\}\\
    f_{\textsf{brochu-4d}}(x)&= & -f_4(x)\\
    f_{\textsf{brochu-6d}}(x)&= & -f_6(x)
  \end{array}  
\]
with $x\in[0,1]^d$, where the minus sign is introduced as we minimize the latent function, while
in~\cite{BDG08} it is maximized. For the definition of the remaining benchmark functions and 
associated bounds on variables the reader is referred to~\cite{Bem19,JY13}.

In all tests, the inverse quadratic RBF with initial parameter $\epsilon=1$ is used
in Algorithm~\ref{algo:idwgopt-pref}, with $\delta=2$ in~\eqref{eq:acquisition},
$N_{\rm init}=\lceil \frac{N_{\rm max}}{3} \rceil$ initial feasible samples
generated by Latin Hypercube Sampling as described in~\cite[Algorithm~2]{Bem19}, and
$\sigma=\frac{1}{N_{\rm max}}$. Self-calibration is executed at steps $N$
indexed by $\mathcal{I}_{\rm{sc}}=\{N_{\rm init},N_{\rm init}+\lceil\frac{N_{\rm max}-N_{rm init}}{4}\rceil,
N_{\rm init}+\lceil\frac{N_{\rm max}-N_{\rm init}}{2}\rceil,N_{\rm init}+\lceil\frac{3(N_{\rm max}-N_{\rm init})}{4}\rceil\}$ over a grid of $10$ values $\epsilon_\ell=\epsilon\theta_\ell$,
$\theta_\ell\in\Theta$, $\ell=1,\ldots,10$, with the same set $\Theta$ used 
to solve problem~\eqref{eq:sasena}.

For comparison, the benchmark problems are also solved by the Bayesian active preference learning algorithm described in~\cite{BDG08}, which is based on a Gaussian Process (GP) approximation of the posterior distribution of the latent preference function $f$. The posterior GP is computed by considering a zero-mean Gaussian process prior, where the prior covariance between the values of the latent function at the two different inputs $x \in \mathbb{R}^n$ and  $y \in \mathbb{R}^n$ is defined by the squared exponential kernel
\begin{align}
\mathcal{K}(x,y) = \sigma_f^2 e^{\frac{    \left\|x-y\right\|^2  }{2\sigma_l^2}}
\end{align}
where $ \sigma_f$ and $\sigma_l$ are positive hyper-parameters. The likelihood describing the observed preferences is constructed by considering the following probability description of the  preference   $\pref(x,y)$:
\begin{align}
\prob(\pref(x,y)|f(x), f(y)) = \left\{\ba{ll} Q\left(\frac{f(y)-f(x)}{\sqrt{2} \sigma_e}\right) & \mbox{if}\ \pref(x,y)=-1 \\
Q\left(\frac{f(x)-f(y)}{\sqrt{2} \sigma_e}\right) &  \mbox{if}\ \pref(x,y)=1 \ea \right.
\end{align}
where $Q$ is the cumulative distribution of the standard Normal   distribution, and $\sigma_e$ is the standard deviation of a zero-mean Gaussian noise which is introduced as a contamination term on the latent function $f$ in order to allow some tolerance on  the preference relations (see~\cite{CG05} for details).  The preference relation $\pref(x,y)=0$ is treated as two independent  observations with preferences $\pref(x,y)=-1$ and $\pref(x,y)=1$. The hyper-parameters $\sigma_f$ and  $\sigma_l$, as well as the noise standard deviation $\sigma_e$, are computed by maximizing the \emph{probability of the evidence}~\cite[Section 2.2]{CG05}. For a fair comparison with the RBF-based algorithm in this paper, these hyper-parameters are re-computed at the  steps indexed by $\mathcal{I}_{\rm{sc}}$. Furthermore, the same number $N_{\rm init}$ of initial feasible samples is generated using Latin hypercube sampling~\cite{MBC79}.

Algorithm~\ref{algo:idwgopt-pref} is executed using both the acquisition
function~\eqref{eq:acquisition} (\textsf{RBF+IDW}) 
and~\eqref{eq:acquisition-improvement-prob} (\textsf{RBF+PI}),
and results compared against those obtained by Bayesian active preference
learning (\textsf{PBO}), using the \emph{expected improvement} as an acquisition function~\cite[Sec. 2.3]{BDG08}.     Results are plotted
in Figures~\ref{fig:benchmarks-1} and~\ref{fig:benchmarks-2}, where
the median performance and the band defined by the best- and worst-case instances
over $N_{\rm test}=20$ runs is reported as a function of the number of queried preferences.
The vertical line represents the last query $N_{\rm init}-1$ at which active preference learning begins.
The average CPU time spent on solving each benchmark problem
is reported in Table~\ref{tab:benchmarks}.

\begin{figure}[p]
\centerline{\includegraphics[width=\hsize]{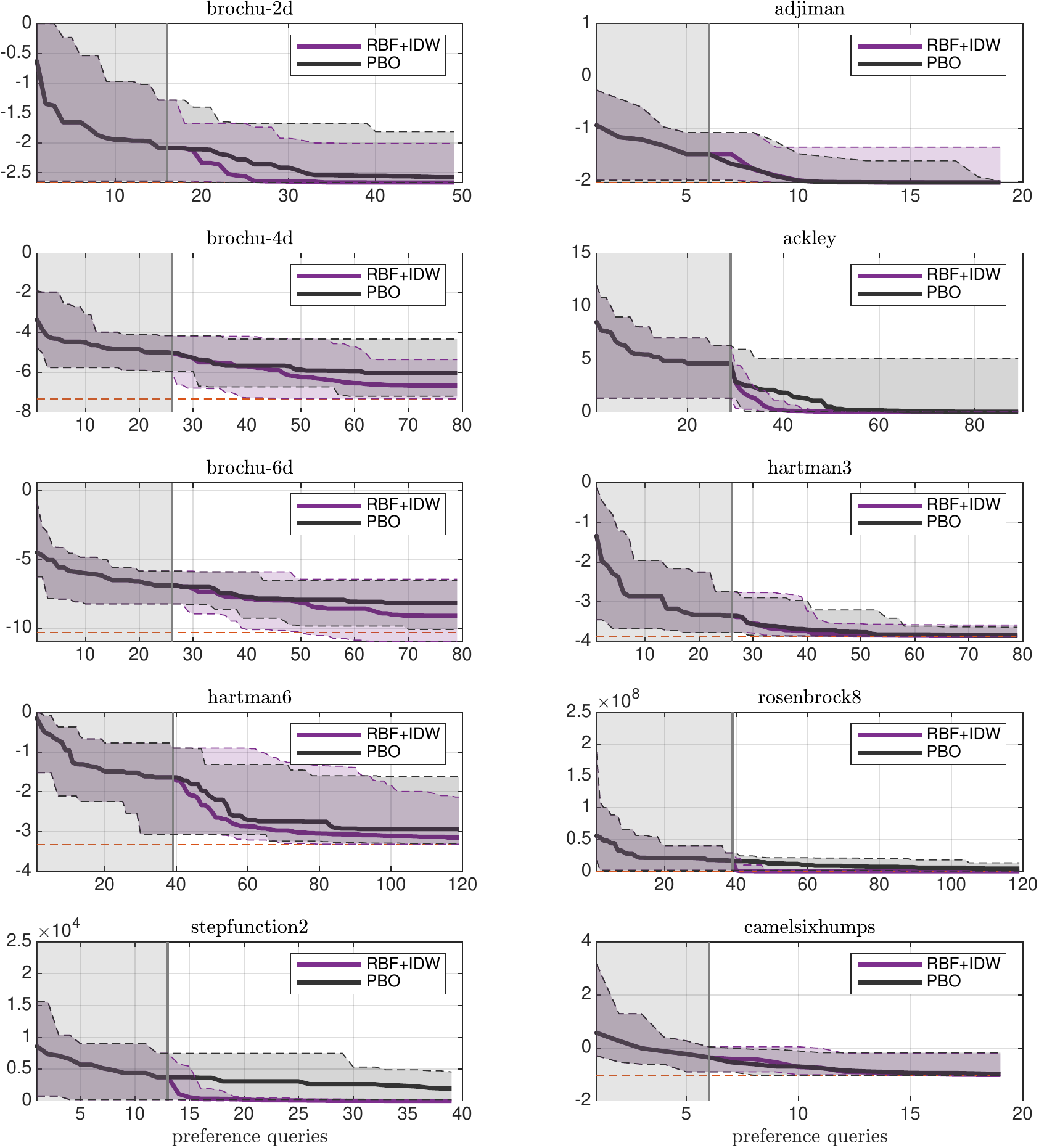}}
\caption{Comparison between Algorithm~\ref{algo:idwgopt-pref} 
based on IDW acquisition~\eqref{eq:acquisition} (\textsf{RBF+IDW}) and
Bayesian preference learning (\textsf{PBO}) on benchmark problems:
median (thick line) and best/worst-case band over $N_{\rm test}=20$ tests. }
\label{fig:benchmarks-1}
\end{figure}

\begin{figure}[p]
\centerline{\includegraphics[width=\hsize]{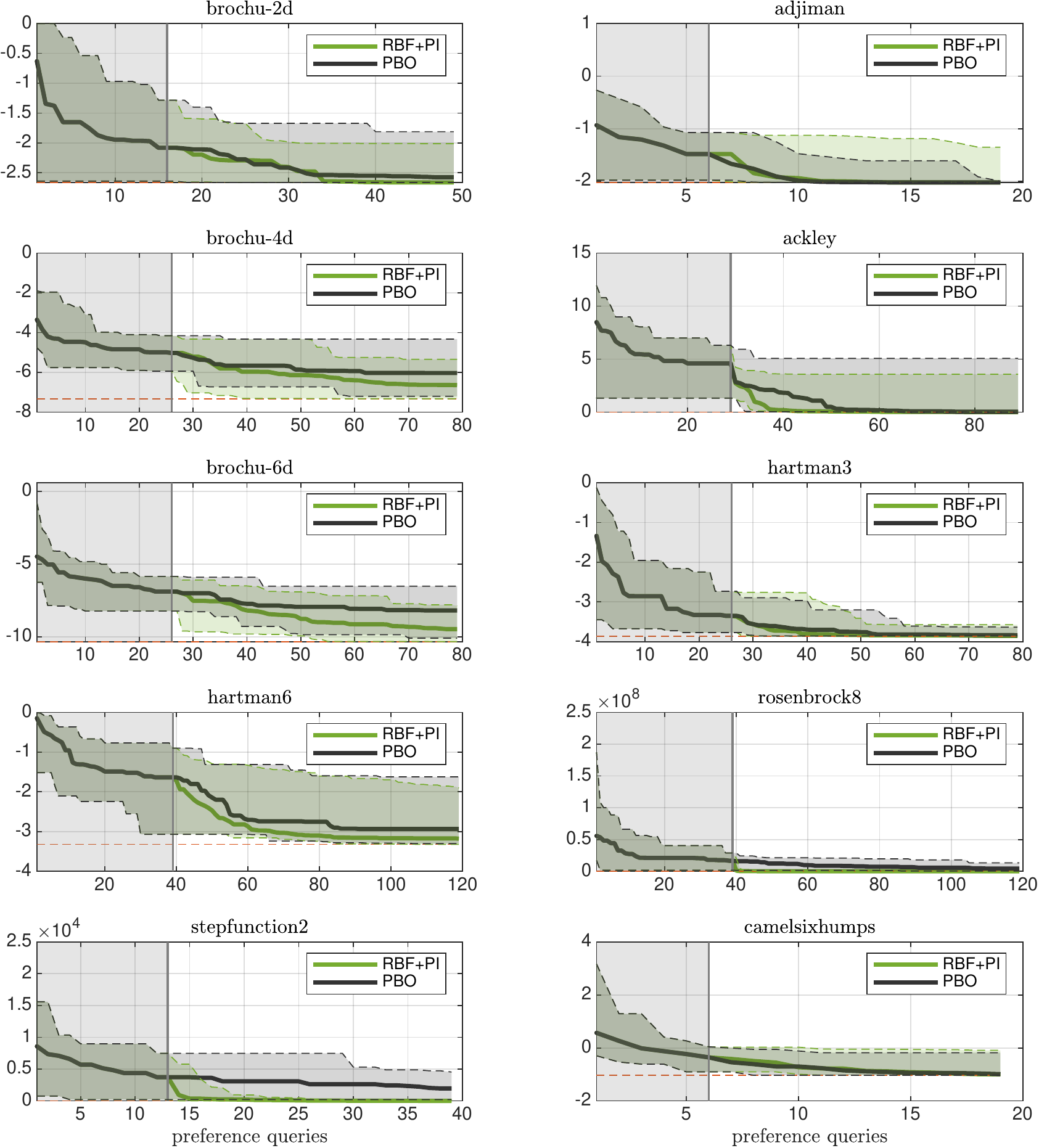}}
\caption{Comparison between Algorithm~\ref{algo:idwgopt-pref} 
based on probability of improvement~\eqref{eq:acquisition-improvement-prob} (\textsf{RBF+PI}) and
Bayesian preference learning (\textsf{PBO}) on benchmark problems:
median (thick line) and best/worst-case band over $N_{\rm test}=20$ tests.. }
\label{fig:benchmarks-2}
\end{figure}

It is apparent that, compared to \textsf{PBO}, the \textsf{RBF+IDW} and \textsf{RBF+PI} 
algorithms perform better in approaching the minimum of the latent function 
and are computationally lighter. The \textsf{RBF+IDW} and \textsf{RBF+PI} algorithms
have instead similar performance and computational load. 
Around 40 to 80\% of the CPU time is spent in self-calibrating $\epsilon$ as 
described in Section~\ref{Self-cal}.

\begin{table}[t]
\begin{center}
\begin{tabular}{l|r|r|r|r}
    problem & $n$ & \textsf{RBF+IDW} & \textsf{RBF+PI} & \textsf{PBO}\\[.5em]\hline
    \textsf{brochu-2d} &  2 & 5.9 & 6.0 & 18.5\\[.5em]\hline
    \textsf{adjiman} &  2 & 1.2 & 1.2 & 13.3\\[.5em]\hline
    \textsf{brochu-4d} &  4 & 21.1 & 21.4 & 30.7\\[.5em]\hline
    \textsf{ackley} &  2 & 30.8 & 30.9 & 51.2\\[.5em]\hline
    \textsf{brochu-6d} &  6 & 20.3 & 22.5 & 32.3\\[.5em]\hline
    \textsf{hartman3} &  3 & 19.7 & 20.4 & 27.2\\[.5em]\hline
    \textsf{hartman6} &  6 & 57.6 & 61.5 & 60.6\\[.5em]\hline
    \textsf{rosenbrock8} &  8 & 68.1 & 70.1 & 306.4\\[.5em]\hline
    \textsf{stepfunction2} &  4 & 4.2 & 4.3 & 45.2\\[.5em]\hline
    \textsf{camelsixhumps} &  2 & 1.2 & 1.2 & 14.6\\
\end{tabular}
\caption{CPU time (s) spent for solving each benchmark problem
considered in the comparison, averaged over $N_{\rm test}=20$ runs.}
\end{center}
\label{tab:benchmarks}
\end{table}

\subsection{Multi-objective optimization by preferences}
\label{sec:mo-results}
We consider the following multi-objective optimization problem
\begin{subequations}
\begin{eqnarray}
    \min_z &&F(z)=\matrice{c}{(2z_1\sin z_2-3\cos(z_1z_2))^2\\
                z_3^2(z_1+z_2)^4\\
                (z_1+z_2+z_3)^2}\label{eq:multi-obj-ex-fun}\\[.5em]
    \st && -1\leq z_i\leq 1,\ i=1,2,3
     \label{eq:multi-obj-ex-constr}
\end{eqnarray}
\label{eq:multi-obj-ex}%
\end{subequations}
Let assume that the preference is expressed by a decision maker in terms of ``similarity'' of
the achieved optimal objectives, that is a Pareto optimal
solution is ``better'' than another one if the objectives $F^\star_1,F^\star_2,F^\star_3$
are closer  to each other. In our numerical tests we therefore mimic the
decision maker by defining a synthetic preference function $\pref$ as in~\eqref{eq:pref_fun-f} 
via the following latent function $f:\rr^n\to\rr$
\begin{equation}
    f(x)=\left\|\matrice{c}{F^\star_1(x)-F^\star_2(x)\\F^\star_1(x)-F^\star_3(x)\\F^\star_2(x)-F^\star_3(x)}\right\|
\label{eq:mo-latent-fcn}
\end{equation}
As we have three objectives, we only optimize over $x_1,x_2$ and set $x_3=1-x_1-x_2$,
under the constraints $x_1,x_2\geq 0$, $x_1+x_2\leq 1$.

Figure~\ref{fig:multi-obj} shows the results obtained by running $N_{\rm test}=20$ times Algorithm~\ref{algo:idwgopt-pref} 
with $\delta=2$, $\epsilon=1$, and the same other settings as in the benchmarks
examples described in Section~\ref{sec:benchmarks}. The optimal scalarization coefficients returned by the algorithm are $x^\star_1=0.2857$, $x_2^\star=0.1952$ and $x_3^\star=1-x_1^\star-x_2^\star=0.5190$, that lead to $F^\star(x^\star)=[1.3921\ 1.3978\ 1.3895]'$. The latent function~\eqref{eq:mo-latent-fcn} 
optimized by the algorithm is plotted in Figure~\ref{fig:multi-obj-latent-fcn}. Note that 
the optimal multi-objective $F^\star$ achieved by setting $x_1=x_2=x_3=\frac{1}{3}$, corresponding
to the intuitive assignment of equal scalarization coefficients, leads to the much
worse result $F^\star=[0.2221\ 0.2581\ 2.9026]'$.

\begin{figure}[t]
\centerline{\includegraphics[width=\hsize]{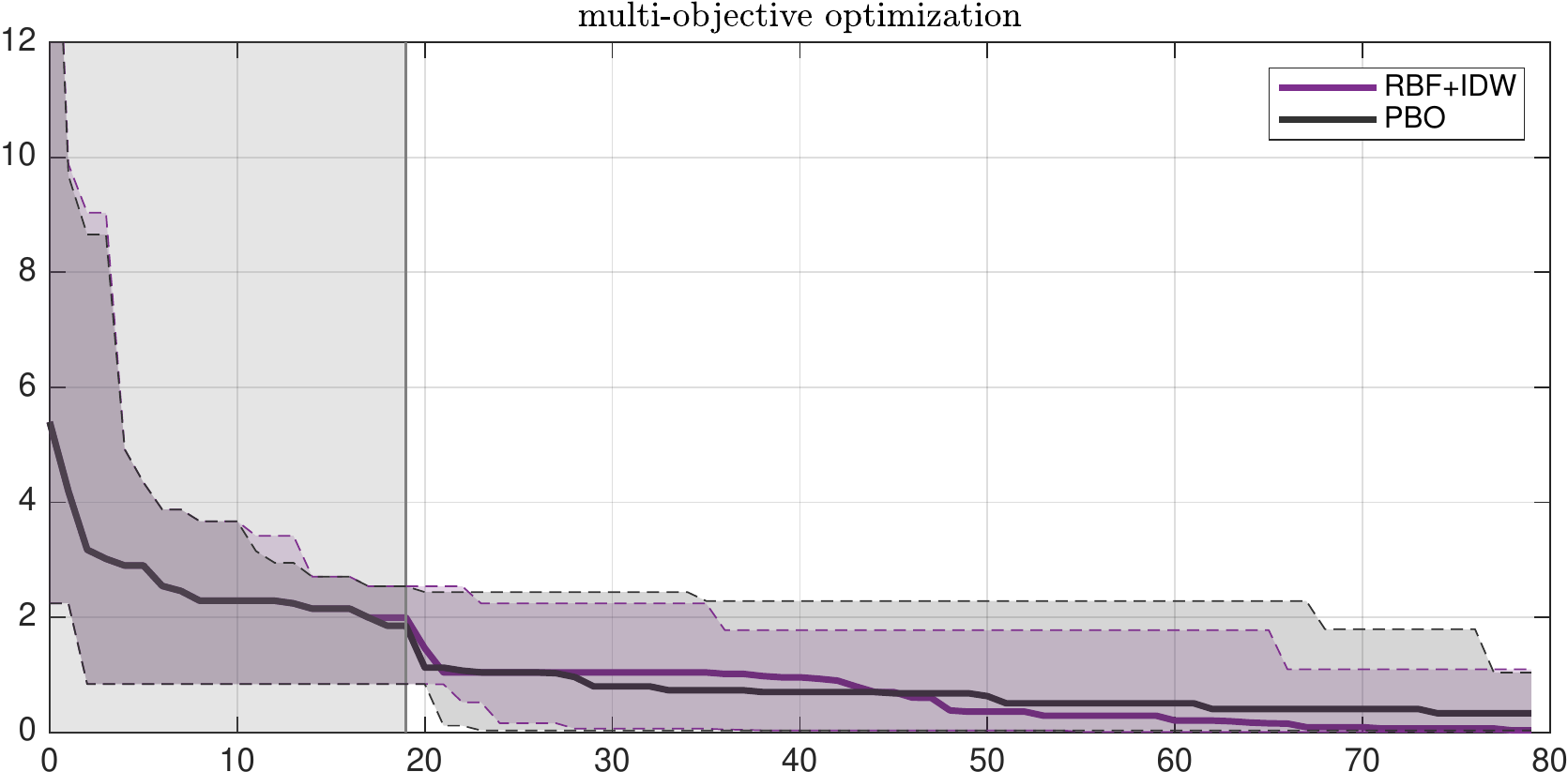}}
\caption{Multi-objective optimization example: median (thick line) and best/worst-case band over $N_{\rm test}=20$ tests of latent function~\eqref{eq:mo-latent-fcn} as a function of
    queried preferences.}
\label{fig:multi-obj}
\end{figure}

\begin{figure}[t]
\centerline{\includegraphics[width=.6\hsize]{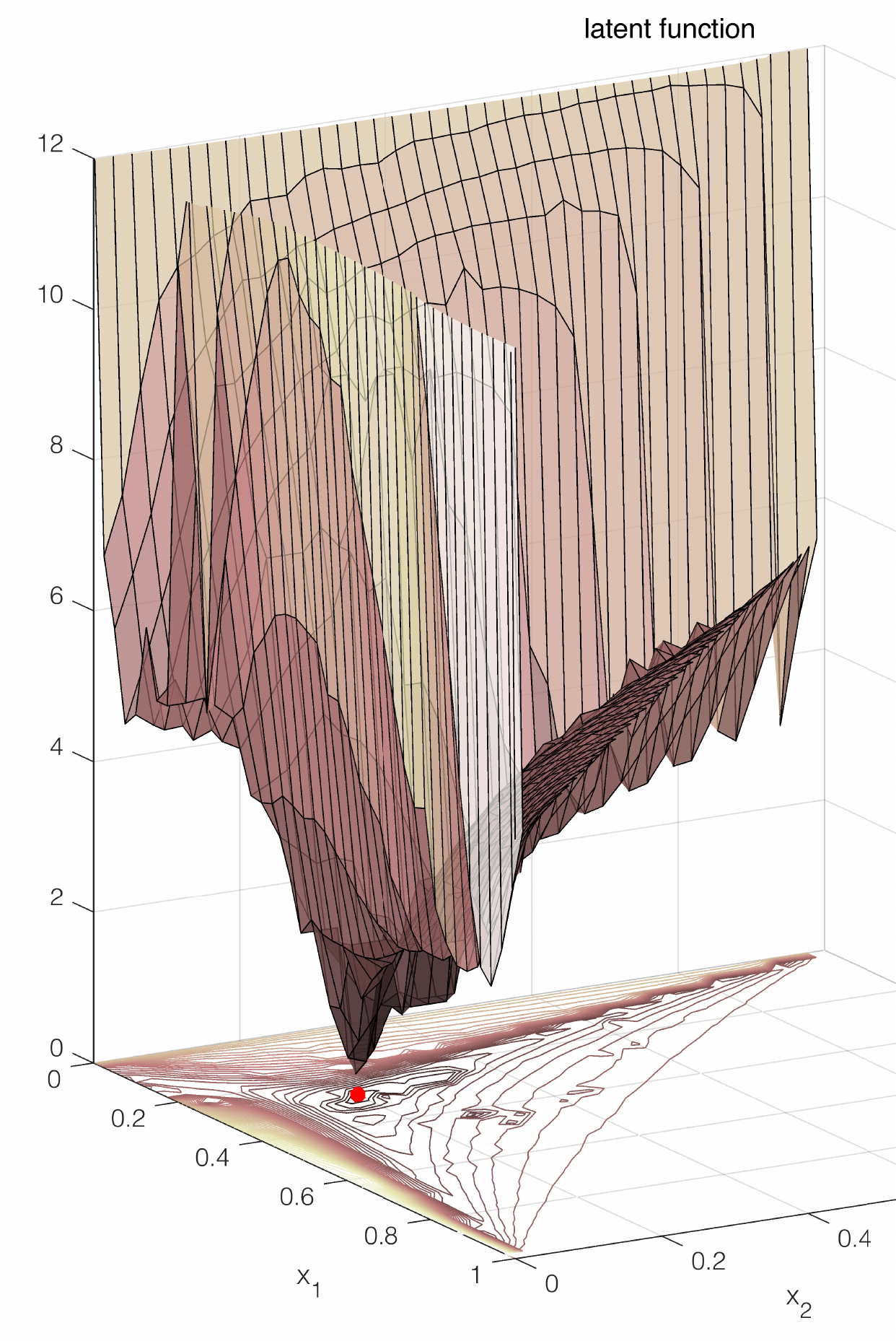}}
\caption{Multi-objective optimization example: latent function~\eqref{eq:mo-latent-fcn}.}
\label{fig:multi-obj-latent-fcn}
\end{figure}

\subsection{Choosing optimal cost-sensitive classifiers via preferences}
We apply now the active preference learning algorithm to solve the problem of choosing optimal   classifiers for object recognition from images when  different costs are associated to different types of misclassification errors. 

A four-class  \emph{convolutional neural network} (CNN) classifier with 3 hidden layers and   a \emph{soft-max} output  layer  is trained using 20000 samples, which consist of  all and only the images of the \emph{CIFAR-10} dataset~\cite{CIFAR} labelled as: \textsf{automobile}, \textsf{deer}, \textsf{frog}, \textsf{ship}, that are referred in the following as classes $\mathcal{C}_1$, $\mathcal{C}_2$, $\mathcal{C}_3$, $\mathcal{C}_4$, respectively. The  network is trained  in 150 epochs using the Adam algorithm~\cite{KB15} and batches of  size 2000, achieving an accuracy of 81\% over a validation dataset of 4000 samples.

We assume that   a decision maker associates different costs to misclassified objects and the predicted class  of an image $\mathcal{U}$   is computed as 
\begin{align} \label{eqn:predclass}
\hat{\mathcal{C}} = \arg\max_{i = \{1,2,3,4\}} x_ip(\mathcal{C}_i|\mathcal{U}) 
\end{align}
where $p(\mathcal{C}_i|\mathcal{U})$ is the network's confidence (namely, the output of the softmax layer) that the image $\mathcal{U}$ is in class $\mathcal{C}_i$, and $x_i$ are nonnegative weights to be tuned in order to take into account the preferences of the  decision maker.  As for the multi-objective
optimization example of Section~\ref{sec:mo-results}, without loss of generality we  set $\sum_{i=1}^{4}x_i=1$ and the constraints  $x_i\geq 0$, $\sum_{i=1}^{3}x_i\leq 1$, thus  eliminating the variable 
$x_{4}=1-\sum_{i=1}^{3}x_i$.

In our numerical tests we  mimic the preferences expressed by the 
decision maker by defining the synthetic  preference function $\pref$ as in~\eqref{eq:pref_fun-f},
where the (unknown) latent function $f:\rr^n\to\rr$ is defined as
\begin{equation}
f(x,d)= \left(1+d\right)\sum_{i=1}^4 \sum_{i=1}^4 C(i,j)r(i,j,x)
\label{eq:NN-latent-fcn}
\end{equation}
In~\eqref{eq:NN-latent-fcn}, the term $r(i,j,x)$  is the number of samples in the validation set of actual class $\mathcal{C}_i$ that are predicted as class $\mathcal{C}_j$  according to the decision rule~\eqref{eqn:predclass}, while  $C(i,j)$ is the cost of  misclassifying a sample of actual class $\mathcal{C}_i$ as class $\mathcal{C}_j$.    The considered costs are reported in Table~\ref{Tab:cost}, which describes  the behaviour of the decision maker in associating a higher cost  in  misclassifying \textsf{automobile} and  \textsf{ship}  rather  than  misclassifying \textsf{deer} and \textsf{frog}. In~\eqref{eq:NN-latent-fcn},  $d$ is a random variable uniformly distributed between $-0.15$ and $0.15$ and it is introduced to  represent a possible inconsistency in the preferences made by the user. 

Figure~\ref{fig:NN} shows the results obtained by running $N_{\rm test}=30$ times Algorithm~\ref{algo:idwgopt-pref} 
with $\delta=2$, $\epsilon=1$, $N_{\rm init}=10$, and the same other settings as in the benchmarks examples described in Section~\ref{sec:benchmarks}, and by running preference-based Bayesian optimization. The optimal weights  returned by the algorithm after evaluating $N_{\rm max}=40$ samples  are $x^\star_1=0.3267$, $x_2^\star=0.1613$, $x_3^\star=0.1944$ and $x_4^\star=1-x_1^\star-x_2^\star-x_3^\star=0.3176$, that lead to a noise-free cost $f(x^\star,0)$ in~\eqref{eq:NN-latent-fcn} equal to $2244$ (against   $f(x^\star,0)=$2585 obtained for unweighted costs, namely, for $x_1=x_2=x_3=x_4=0.25$). As expected, higher weights are associated to \textsf{automobile} and \textsf{ship} (class $\mathcal{C}_1$ and $\mathcal{C}_4$, respectively).    For judging the quality of 
the computed solution, the minimum of the   noise-free cost $f(x^\star,0)=2201$ is  computed by PSO
and used as the reference global optimum. 

\begin{figure}[t]
    \begin{center}{\includegraphics[width=\hsize]{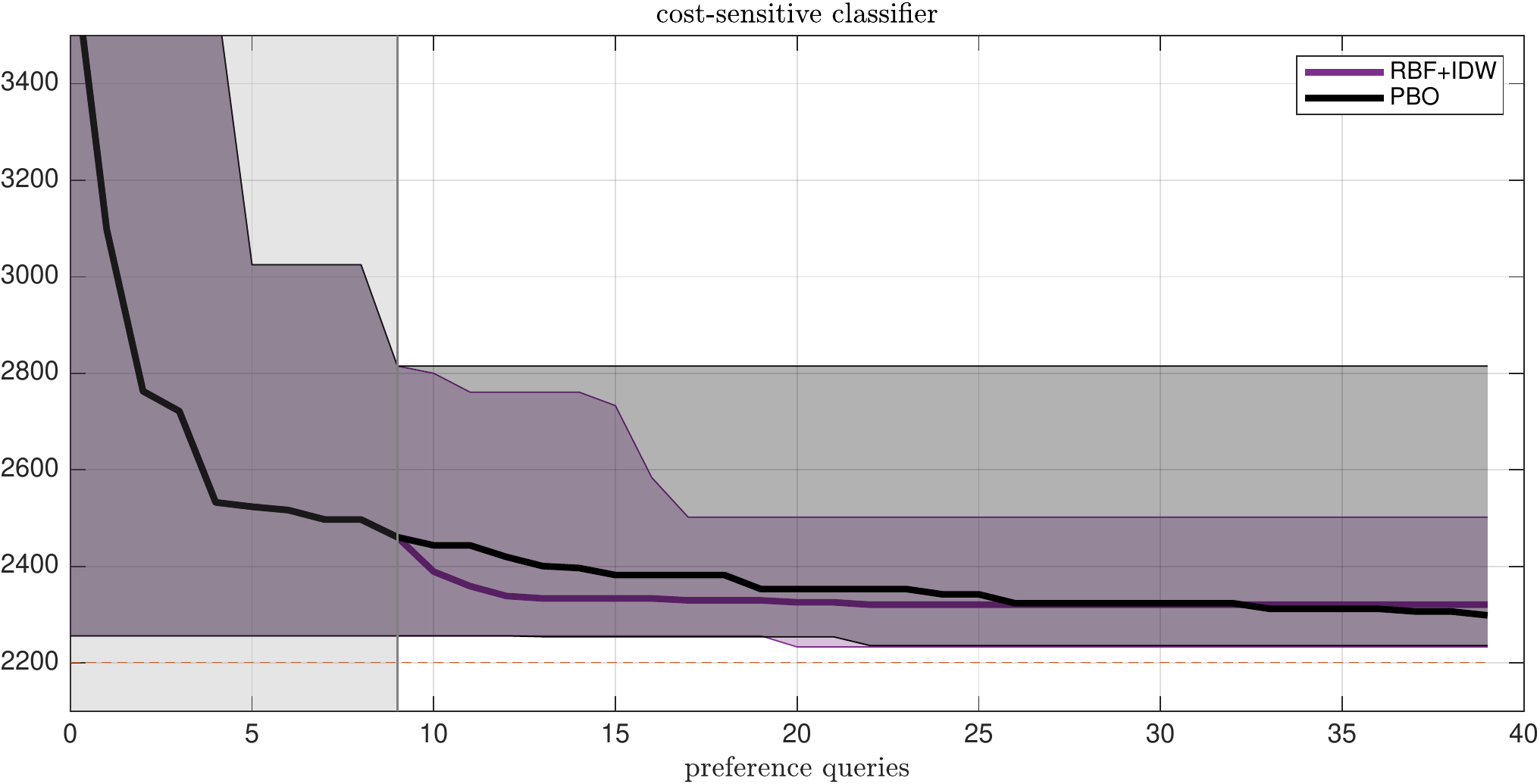}}\end{center}
    \caption{Noise-free cost $f(x,0)$ as a function of the number of queried preferences.  Median (solid lines) and bands defined by the best- and worst-case instances over $N_{\rm test}=30$; reference global optimum achieved by PSO (dashed red line). 
        .}%
    \label{fig:NN}
\end{figure} 

\begin{table}
    \begin{center}
        \begin{tabular}{|c|c|c|c|c|c|}
            \cline{3-6}
            \multicolumn{2}{c|}{} & \multicolumn{4}{c|}{predicted class} \\
            \cline{3-6}
            \multicolumn{2}{c|}{} & $\mathcal{C}_1$ & $\mathcal{C}_2$ & $\mathcal{C}_3$ & $\mathcal{C}_4$\\
            \hline
            \multirow{4}{*}{\begin{sideways} actual     \end{sideways}} \multirow{4}{*}{\begin{sideways} class  \end{sideways}} & $\mathcal{C}_1$ & 0 & 10 & 10 & 3 \\
            \cline{2-6}
            & $\mathcal{C}_3$ & 4 & 0 & 2 &  4 \\
            \cline{2-6}
            & $\mathcal{C}_3$ & 4 & 2 &0 &  4 \\
            \cline{2-6}
            & $\mathcal{C}_4$ & 3 & 10 & 10 & 0 \\
            \hline
        \end{tabular}
    \end{center}
    \caption{Cost matrix.}
    \label{Tab:cost}
\end{table}

\section{Conclusions}
\label{sec:conclusions}
In this paper we have proposed an algorithm for choosing the vector of decision variables
that is best in accordance with pairwise comparisons with all possible other values. Based
on the outcome of an incremental number of comparisons between given samples of the decision vector, the main idea is to attempt learning a latent cost function, using radial basis function interpolation, that, when compared at such samples, provides the same preference outcomes.
The algorithm actively learns such a surrogate function by proposing iteratively a new sample
to compare based on a trade-off between minimizing the surrogate and visiting
areas of the decision space that have not yet been explored. Through several numerical tests,
we have shown that the algorithm performs better than active preference learning
based on Bayesian optimization, in that it approaches the optimal decision vector with less
computations.

The approach can be extended in several directions. First, rather than only comparing the new sample $x_{N+1}$ with the current best $x^\star$, one could ask for expressing
preferences also with one or more of the other existing samples $x_1,\ldots,x_N$. Second,
the codomain of the comparison function $\pref(x,y)$ could be extended to say $\{-2,-1,0,1,2\}$ 
where $\pref(x,y)=\pm2$ means ``$x$ is much better/worse than $y$'', and then  extend~\eqref{eq:RBF-pref}
to include a much larger separation than $\sigma$ whenever the corresponding preference $\pref=\pm2$.
Third, often one can qualitatively assess whether a given sample $x$ is ``very good", ``good", ``neutral'', ``bad", or ``very bad", and take this additional information into account when learning the surrogate function, 
for example by including additional constraints that force the surrogate function to 
lie in $[0,0.2]$ on all ``very bad'' samples, in $[0.2,0.4]$ on all ``bad'' samples, \ldots, in $[0.8,1]$ on
all ``very good'' ones, and choosing an appropriate value of $\sigma$. Furthermore, while a certain tolerance
to errors in assessing preferences is built-in in the algorithm thanks to the use of slack variables in~\eqref{eq:RBF-pref}, the approach could be extended to better take evaluation errors into account in the overall formulation and solution method.

Finally, we remark that one should be careful in using the learned surrogate function to extrapolate preferences on arbitrary new pairs of decision vectors, as the learning process is tailored to detecting
the optimizer rather than globally approximating the unknown latent function and, moreover,
the chosen RBFs may not be adequate to reproduce the shape of the unknown latent function.

\section*{Acknowledgement}
The authors thank Luca Cecchetti for pointing out the literature references in psychology and neuroscience cited in the introduction of this paper.

\bibliographystyle{plain}
\bibliography{global_idw_pref}

\section*{Appendix}
\subsection*{Proof of Theorem~\ref{Th:eqprob}} 
Let $\beta(\lambda)$ be the minimizer of problem~\eqref{eqn:solveP} for some positive scalar  $\lambda$. Let us define  $\tau(\lambda)=\left\| \beta(\lambda) \right\|$ and the set $B_{\tau} = \{\beta \in \mathbb{R}^N: \|\beta \|=\tau(\lambda) \}$. Then, we have 
\begin{align*}
\beta(\lambda)   = &  \arg\min_{\beta \in \mathbb{R}^N} \sum_{h=1}^{M} c_{h} \ell_{b_{h}}(\Phi(\epsilon,X, x_{i(h)}, x_{j(h)})'\beta) + \frac{\lambda}{2} \left\| \beta \right\|^2 \\
= &  \arg\min_{\beta \in B_{\tau}}  \sum_{h=1}^{M} c_{h} \ell_{b_{h}}(\Phi(\epsilon,X, x_{i(h)}, x_{j(h)})'\beta)   \\
= &  \tau(\lambda)\arg\min_{u: \|u\|=1}  \sum_{h=1}^{M} c_{h} \ell_{b_{h}}(\tau(\lambda)\Phi(\epsilon,X, x_{i(h)}, x_{j(h)})'u) \\
= &  \tau(\lambda) u^\star
\end{align*}
where $u^\star$ is the minimizer of~\eqref{eqn:utau}. Thus, $u^\star=\frac{\beta(\lambda)}{\tau(\lambda)}$.  

\subsection*{Proof of Theorem~\ref{Th:int}}

Let $u, \bar u \in \mathbb{R}^N$ be arbitrary unit vectors. Then, there exists an orthogonal (rotation) matrix $R$ with determinant $+1$  such that $\bar u = R'u$. Let $\varphi: \mathbb{R}^N \rightarrow \mathbb{R}^N$ be a vector value function defined as $\varphi(v)=Rv$. Note that the the Jacobian matrix $J_{\varphi}$ of $\varphi$ is $R$, and thus its determinant $det(J_{\varphi})$ is equal to $+1$. 

Let us now write the integral 	$I_t(\bar c_t,\tau,u)$ in~\eqref{eqn:inty} as
\begin{subequations}
	\begin{align}
	I_t(\bar c_t,\tau,u) = & \int_{ \Phi \in \mathbb{R}^N}   e^{-\bar c_{t}\ell_{t}(\tau \Phi'u)}\kappa(\Phi)  d\Phi  \\
	= & \int_{ v \in \mathbb{R}^N}   e^{-\bar c_{t}\ell_{t}(\tau \varphi(v)'u)}  \kappa(\varphi(v)) det(J_{\varphi}) dv \\
	= &  \int_{ v \in \mathbb{R}^N}   e^{-\bar c_{t}\ell_{t}(\tau v'R'u)}  \kappa(v)   dv \label{eqn:P:Ib2} \\
		= &  \int_{ v \in \mathbb{R}^N}   e^{-\bar c_{t}\ell_{t}(\tau v'\bar u)}  \kappa(v)   dv \label{eqn:P:Ib3} \\
	= & 	I_t(\bar c_t,\tau,\bar u)  
	\end{align}
\end{subequations}
where~\eqref{eqn:P:Ib2} holds since $\kappa(\varphi(v))=e^{-\varphi(v)'\varphi(v)}=e^{-v'RR'v}=e^{-v'v} = \kappa(v)$. 

\end{document}